\documentclass[lettersize,journal]{IEEEtran}
\usepackage{graphicx}
\usepackage[T1]{fontenc}
\usepackage[utf8]{inputenc}
\usepackage{microtype}
\usepackage{booktabs}
\usepackage{amssymb} 
\usepackage[colorlinks=true,linkcolor=black,citecolor=black,urlcolor=black]{hyperref}
\usepackage{amsthm}
\usepackage{amsmath,amsfonts}
\usepackage{algpseudocode}
\usepackage{algorithm}
\usepackage{array}
\usepackage[caption=false,font=normalsize,labelfont=sf,textfont=sf]{subfig}
\usepackage{textcomp}
\usepackage{stfloats}
\usepackage{url}
\usepackage{verbatim}
\usepackage{cite}
\usepackage{subcaption}
\usepackage{caption}
\newtheorem{assumption}{Assumption}
\newtheorem{definition}{Definition}
\newtheorem{theorem}{Theorem}
\newtheorem{lemma}{Lemma}

\DeclareMathOperator*{\argmin}{argmin}

\begin{document}

\title{Augmented Lagrangian Multiplier Network for State-wise Safety in Reinforcement Learning}

\author{Jiaming Zhang,
        Yujie Yang,
        Yao Lyu, 
        Shengbo Eben Li*,
        Liping Zhang*
\thanks{This study is supported by National Natural Foundation of China with 12571323, National Science and Technology Major Project (No 2025ZD1606200), Beijing Natural Science Foundation with L257002 and NSF China with 92582205. It is also partially supported by SunRising AI Lab.}
\thanks{* Corresponding authors: Liping Zhang and Shengbo Eben Li.}

\thanks{Jiaming Zhang and Liping Zhang are with the Department of Mathematical Sciences, Tsinghua University, Beijing, 100084, China (e-mail: jiaming-24@mails.tsinghua.edu.cn, lipingzhang@tsinghua.edu.cn).}

\thanks{Yujie Yang and Yao Lyu are with the School of Vehicle and Mobility, Tsinghua University, Beijing, 100084, China (e-mail: yangyj21@mails.tsinghua.edu.cn, lvyao@mail.tsinghua.edu.cn).}

\thanks{Shengbo Eben Li is with the School of Vehicle and Mobility \& College of AI, Tsinghua University, Beijing, 100084, China (e-mail: lishbo@tsinghua.edu.cn).}
}
\maketitle

\begin{abstract}
Safety is a primary challenge in real-world reinforcement learning (RL). 
Formulating safety requirements as state-wise constraints has become a prominent paradigm. Handling state-wise constraints with the Lagrangian method requires a distinct multiplier for every state, necessitating neural networks to approximate them as a multiplier network. However, applying standard dual gradient ascent to multiplier networks induces severe training oscillations. This is because the inherent instability of dual ascent is exacerbated by network generalization---local overshoots and delayed updates propagate to adjacent states, further amplifying policy fluctuations. Existing stabilization techniques are designed for scalar multipliers, which are inadequate for state-dependent multiplier networks.
To address this challenge, we propose an augmented Lagrangian multiplier network (ALaM) framework for stable learning of state-wise multipliers. ALaM consists of two key components. First, a quadratic penalty is introduced into the augmented Lagrangian to compensate for delayed multiplier updates and establish the local convexity near the optimum, thereby mitigating policy oscillations. Second, the multiplier network is trained via supervised regression toward a dual target, which stabilizes training and promotes convergence. Theoretically, we show that ALaM guarantees multiplier convergence and thus recovers the optimal policy of the constrained problem. Building on this framework, we integrate soft actor-critic (SAC) with ALaM to develop the SAC-ALaM algorithm. Experiments demonstrate that SAC-ALaM outperforms state-of-the-art safe RL baselines in both safety and return, while also stabilizing training dynamics and learning well-calibrated multipliers for risk identification.

\end{abstract}

\section{Introduction}
\IEEEPARstart{R}{einforcement} learning has demonstrated remarkable performance in solving complex sequential decision-making and optimal control problems. Its enormous potential has driven widespread adoption across various fields, including autonomous driving \cite{li2023reinforcement}, embodied artificial intelligence \cite{gupta2021embodied}, and large language models \cite{rafailov2023direct}. Fundamentally, RL enables agents to achieve optimal performance by learning policies that maximize expected cumulative reward. However, optimizing solely for return is frequently inadequate for real-world deployments. For instance, autonomous vehicles cannot risk collisions simply to minimize travel time. Therefore, it is necessary to incorporate cost signals and feasibility constraints into the training process to ensure policy safety \cite{li2023reinforcement}.

Traditionally, safe RL is formulated as a Constrained Markov Decision Process (CMDP), defining safety requirements by bounding the expected cumulative cost \cite{ray2019benchmarking}. However, this paradigm only bounds the cumulative cost of an entire trajectory, failing to guarantee safety at every individual state. An agent may still violate constraints at specific steps while keeping its total trajectory cost below the permissible threshold. To overcome this limitation, recent methods adopt state-wise constraints to enforce feasibility at every state \cite{ma2021model, ganai2023iterative, yu2023safe}. These constraints provide a more rigorous guarantee, ensuring the agent remains confined to the safe region at every time step \cite{yang2023feasible}.

To solve safe RL problems, existing approaches primarily combine RL algorithms with constrained optimization techniques, among which the Lagrange multiplier method is the most widely used. For example, algorithms such as TRPO-Lag, PPO-Lag, and SAC-Lag embed classic RL methods into Lagrangian formulation to solve CMDPs \cite{ray2019benchmarking, paternain2019constrained}.
RCPO directly integrates the constraint function into the reward to form a Lagrangian composite signal, and employs a multi-timescale update scheme with theoretical convergence guarantees \cite{tessler2019reward}. To restrict the probability of constraint violations under extreme scenarios, the Lagrangian framework is combined with Conditional Value-at-Risk (CVaR) and chance constraints in \cite{chow2018risk}. This work derives the policy gradient for CVaR MDPs and develops specific actor-critic algorithms.
These methods convert the original formulation into a primal-dual minimax problem to balance task performance and safety requirements. Optimization proceeds with an alternating update scheme: the policy optimizes the Lagrangian objective to maximize returns alongside penalized costs, while the multiplier updates via dual gradient ascent to dynamically calibrate the penalty scale. Consequently, this multiplier increases to enforce safety compliance when constraint thresholds are exceeded, and decreases to prioritize reward maximization once safety is guaranteed.
While these methods effectively solve trajectory-based CMDPs using a single scalar multiplier, they are fundamentally ill-equipped for state-wise constraints. Enforcing safety across a continuous state space introduces an infinite number of constraints, which theoretically requires an infinite number of distinct multipliers. A naive workaround is to mathematically aggregate these infinitely many constraints into a single global condition, such as summing the constraint violations across all states, and penalize it with a single scalar multiplier. However, this aggregation essentially assigns a uniform penalty across the entire state space, completely discarding the state-specific nature of the constraints. This homogeneity significantly increases optimization difficulty and usually yields suboptimal behaviors in practice: the policy becomes overly conservative in safe regions while remaining unsafe in hazardous states.

To genuinely satisfy state-wise constraints, the algorithm must maintain distinct multipliers for different states. In continuous state spaces, this is practically achieved by employing parameterized approximation functions, transforming the theoretical infinite multipliers into a neural network. This multiplier network takes the state as input and dynamically outputs a dedicated penalty for that specific state.
For instance, FAC trains a multiplier network to assign specific penalties to individual states \cite{ma2021feasible}. To accommodate local safety constraints, this research defines a state-wise Lagrange function. System optimization then proceeds through the primal-dual procedure over this formulation: the policy minimizes the Lagrangian, and the multiplier parameters adjust via dual gradient ascent.
However, empirical results in FAC exhibit noticeable training oscillations, a drawback explicitly acknowledged by the authors as a primary limitation of their method. In fact, this instability is a broadly recognized challenge even in conventional Lagrangian-based safe RL algorithms that rely on a single scalar multiplier \cite{pmlr-v119-stooke20a, peng2022model}. The root cause lies in the dynamics of standard dual gradient ascent: the adjustments to the multiplier inherently lag behind the rapid shifts in the policy's safety behavior. This persistent delay traps the multiplier in a cycle of insufficient penalization and severe overshooting, alternately leading to constraint violations and overly conservative policies.
To overcome this, the PID-Lagrangian approach models the multiplier update as a control problem, replacing basic gradient ascent with a proportional-integral-derivative controller to dampen these oscillations \cite{pmlr-v119-stooke20a}. Nonetheless, it introduces extra hyperparameters, lacks theoretical guarantees, and frequently fails to outperform standard gradient ascent in practice \cite{spoor2025empirical}. Alternatively, APPO employs the augmented Lagrangian method, which adds a quadratic penalty to enable closed-form multiplier updates \cite{dai2023augmented}.
While APPO has been validated on scalar multipliers, our experiments reveal that its efficacy in mitigating oscillations within multiplier networks is limited and often sacrifices task performance. This inadequacy arises because introducing a parameterized network intrinsically aggravates the multiplier's oscillation problem, driven by two main factors. First, network generalization can erroneously over-penalize safe regions near hazards. Second, policy updates induce distribution shifts that trigger catastrophic forgetting, mistakenly erasing valid penalties for still-hazardous states. To date, achieving stable training for multiplier networks under state-wise constraints remains an open challenge.

To address this challenge, we propose an augmented Lagrangian multiplier network (ALaM) framework for stable learning of state-wise multipliers in safe reinforcement learning. By integrating soft actor-critic (SAC) with ALaM, we further develop SAC-ALaM, a practical safe RL approach. To the best of our knowledge, our approach is the first to successfully stabilize the training of parameterized multipliers. Our main contributions are summarized as follows.

\begin{itemize}
    \item We introduce an augmented Lagrangian formulation with a quadratic penalty for constraint violations. This design compensates for delayed multiplier responses in conventional methods through instantaneous feedback, thereby mitigating policy oscillations. Meanwhile, the quadratic penalty establishes local convexity around the optimum. This structural refinement improves the optimization landscape and alleviates the ill-conditioned saddle-point geometries in the original Lagrangian formulation, preventing policy divergence and therefore leading to substantially more stable training.

    \item We show that direct gradient ascent on multiplier network parameters, as commonly used in existing methods, does not realize true dual ascent in function space because nonlinear parameterization distorts the effective update step. To address this issue, we train the multiplier network via supervised regression toward an analytical dual target. This mechanism more faithfully captures the desired multiplier function and yields improved training stability and convergence behavior.

    \item We provide rigorous theoretical analysis showing that ALaM guarantees multiplier convergence under standard assumptions. Based on this result, we further establish that any limit point of the induced policy sequence is both feasible and optimal for the constrained problem. These results provide a principled foundation for stable multiplier network training in safe RL.

\end{itemize}

Finally, extensive experiments demonstrate that SAC-ALaM achieves state-of-the-art performance in balancing task return and safety compliance. In addition, it learns well-calibrated state-dependent multipliers for reliable local risk identification and exhibits strong generalization to unseen scenarios.

\section{Preliminaries}
\subsection{Reinforcement learning with state constraints}
RL is typically formulated as a Markov Decision Process (MDP) represented by $\mathcal{M} = \langle \mathcal{X}, \mathcal{A}, f, r, \gamma,\mu \rangle$. Here, $\mathcal{X}$ and $\mathcal{A}$ denote the state and action spaces. 
$f: \mathcal{X} \times \mathcal{A} \rightarrow \mathcal{X}$ specifies the state transition dynamics, where $x_{t+1} = f(x_t, a_t)$ defines the next state given the current state $x_t$ and action $a_t$.
Furthermore, $r: \mathcal{X} \times \mathcal{A} \rightarrow \mathbb{R}$ is the reward function, $\gamma \in [0, 1]$ is the discount factor, and $\mu$ represents the initial state distribution. The policy in the MDP considered in this paper is a deterministic function that maps a state to an action: $\pi:\mathcal{X}\rightarrow\mathcal{A}$. Under this policy, the expected cumulative reward $J(\pi)$ is defined as:
\begin{equation*}
    J(\pi)=\mathbb{E}_{x_0 \sim \mu}\left [\sum_{t=0}^{\infty}\gamma^tr(x_t, \pi(x_t))\right].
\end{equation*}
This function is the maximization objective in standard RL.

However, optimizing solely for performance is insufficient for real-world applications; the safety of the policy must also be guaranteed. We formalize this feasibility requirement through state-wise constraints $h(x) \le 0$, where $h$ is a function whose zero sublevel set defines the safe region. Since RL involves sequential decision-making, we enforce this constraint at every time step, i.e., $h(x_t) \le 0,\; \forall t \ge 0$. Consequently, the safe RL problem is formulated as 
\begin{equation} \label{eq:1}
\begin{aligned}
    \max_{\pi} \quad & J(\pi) \\
    \text{s.t.} \quad & h(x_t) \le 0, \quad \forall x_0 \sim \mu, \; t \ge 0.
\end{aligned}
\end{equation}

\subsection{Feasibility function}
Directly enforcing state-wise constraints is computationally intractable, as it involves an infinite number of conditions spanning the continuous state space and the infinite time horizon. To overcome this, we adopt the feasibility function \cite{yang2023feasible, yang2023synthesizing}, which aggregates the infinite-horizon constraints into a single condition determined by the initial state. We formally define it as follows:

\begin{definition}[Feasibility Function]
$F^\pi : \mathcal{X} \to \mathbb{R}$ is a feasibility function for policy $\pi$ if:
\begin{enumerate}
    \item For any initial state $x_0 = x$, the condition $F^\pi(x) \le 0$ holds if and only if $h(x_t) \le 0$ for all $t \ge 0$.
    \item There exists a risky self-consistency operator $\mathcal{T}^\pi : \mathbb{R}^{\mathcal{X}} \to \mathbb{R}^{\mathcal{X}}$ such that:
    \begin{enumerate}
        \item[(i)]  $\mathcal{T}^\pi F^\pi = F^\pi$.
        \item[(ii)]  $\mathcal{T}^\pi$ is a monotone operator.
        \item[(iii)]  $\mathcal{T}^\pi$ is a $\gamma$-contraction under the infinity norm. That is, for any $F_1, F_2 \in \mathbb{R}^{\mathcal{X}}$, $\|\mathcal{T}^\pi F_1 - \mathcal{T}^\pi F_2\|_\infty \le \gamma \|F_1 - F_2\|_\infty$, where $\gamma \in [0, 1)$ is the discount factor.
    \end{enumerate}
\end{enumerate}
\end{definition}

As a concrete example, the cost value function is a standard feasibility function:
\begin{equation*}
    F^{\pi}(x)=\sum_{t=0}^{\infty}\gamma^tc(x_t),\quad \forall x \in \mathcal{X},
\end{equation*}
where $x_0 = x$, $x_{t+1} = f(x_t, \pi(x_t))$, and $c(x) = \mathbf{1}_{h(x)>0}$ is the indicator function of $h$. Since $c(x)$ is non-negative, $F^\pi(x)\le 0$ holds if and only if the state $x$ is long-term feasible under policy $\pi$. The corresponding risky self-consistency operator is $\mathcal{T}^\pi F^\pi(x)=c(x)+\gamma F^\pi(f(x,\pi(x)))$.
Other commonly used feasibility functions include the constraint decay functions and Hamilton-Jacobi reachability functions (e.g., see \cite{yang2026feasibility} for a detailed review).

By utilizing this feasibility function, we effectively bypass the computationally prohibitive infinite-horizon conditions while strictly guaranteeing state safety. Consequently, the original safe RL problem can be equivalently reformulated as follows:
\begin{equation} \label{eq:2}
\begin{aligned}
    \max_{\pi} \quad & J(\pi) \\
    \text{s.t.} \quad & F^{\pi}(x) \le 0,\quad \forall x \in \text{supp}(\mu).
\end{aligned}
\end{equation}
Here, $\text{supp}(\mu)=\{x\in\mathcal{X}:\mu (x)>0\}$ denotes the support of the initial state distribution $\mu$. While the feasibility function converts the infinite-horizon constraints into single-step conditions, the formulation still imposes an infinite number of constraints across the state space. In contrast, the standard CMDPs evaluate only the expected cost under the distribution $\mu$, and reduce the safety requirement to a single global constraint. Therefore, the state-wise constrained framework provides a more rigorous definition of safety.

\subsection{Lagrange multiplier network}
Under the standard CMDP setting, safe RL aims to maximize expected returns and restrict expected costs, i.e., $J_c(\pi):=\mathbb{E}_{x_0 \sim \mu}\left [\sum_{t=0}^{\infty}\gamma^tc(x_t, \pi(x_t))\right] \le d$. Notably, $c$ herein denotes the environmental cost signal, which is distinct from the feasibility indicator function defined previously. Within the Lagrange multiplier paradigm, the dual problem with a scalar multiplier $\lambda \ge 0$ takes the following standard structure:
\begin{equation*}
\max_{\lambda \ge 0} \min_{\theta} L(\theta, \lambda) = -J(\pi_\theta) + \lambda (J_c(\pi_\theta) - d).
\end{equation*}
The algorithm updates parameters through an alternate optimization mechanism. Specifically, the policy minimizes the Lagrangian function, and the multiplier adjusts via the dual gradient ascent step:
\begin{equation*}
\theta \leftarrow \argmin_{\theta} L(\theta, \lambda),\quad \lambda \leftarrow \lambda + \alpha_\lambda (J_c(\pi) - d).
\end{equation*}
This setup aggregates the risks over the entire state distribution into a single global constraint, which fails to guarantee the local safety of specific states.

To establish genuine state-wise safety, the system must satisfy an infinite number of local constraints. In continuous state spaces, it is impractical to maintain an infinite number of independent multipliers. Therefore, recent advanced methods introduce a parameterized multiplier network to approximate the theoretical multipliers \cite{ma2021feasible}. Based on this network, the framework reformulates the Lagrangian to accommodate the state-wise cost value constraints:
\begin{equation*}
L(\theta, w) = -J(\pi_\theta) + \mathbb{E}_{x \sim \mu}\lambda_w(x)(v_c^{\pi_\theta}(x) - d).
\end{equation*}
Under this formulation, the multiplier parameters $w$ adjust via dual gradient ascent based on constraint violations. The corresponding primal-dual updates are:
\begin{equation}\label{eq:SAC-LagNet}
\begin{aligned}
\theta &\leftarrow \argmin_{\theta} L(\theta, w),\\
 w &\leftarrow w + \alpha_w\mathbb{E}_{x\sim \mathcal{B}}\left[(v_c^{\pi_\theta}(x) - d)\nabla_w \lambda_w(x)\right].
\end{aligned}
\end{equation}

Although the multiplier network theoretically equips the algorithm with the capacity to control local risks at each state, the transition from theoretical scalar multipliers to a neural network reshapes the optimization landscape. This is because the gradients of the neural network distort the exact theoretical step size within the functional space. Furthermore, during the alternate update of policy and multiplier parameters, the generalization of multiplier parameters amplifies the oscillations intrinsic to traditional primal-dual methods. This instability constitutes the exact problem that this research aims to resolve.

\section{Augmented Lagrangian multiplier network}
To overcome the instability in multiplier network optimization, we introduce the augmented Lagrangian method (ALM) \cite{bertsekas2014constrained}. ALM augments the standard Lagrange function with a quadratic penalty for constraint violations, which compensates for the inherent response delays and parameter overshoots in traditional methods via instantaneous penalty feedback, thereby mitigating algorithmic oscillations. Furthermore, ALM establishes local convexity around the optimum and reshapes the optimization landscape into a stable basin. This structural refinement resolves the ill-conditioned saddle-point geometries in the Lagrangian formulation, prevents the fluctuations caused by policy divergence and therefore improves training stability. Motivated by these advantages, we develop the augmented Lagrangian multiplier network method to solve problem \eqref{eq:2}.

\subsection{Augmented Lagrangian function}
To construct the augmented framework for problem \eqref{eq:2}, we first address the continuum of state-wise inequality constraints. We introduce a state-dependent, non-negative slack function $p(x) \ge 0$ to convert these inequalities into equivalent equality constraints:
\begin{equation*}
\begin{aligned}
     \max_{\pi,p}\quad &J(\pi)\\
     \text{s.t.}\quad  &F^{\pi}(x) + p(x)=0,\quad p(x) \ge 0, \forall x \in \text{supp}(\mu).
\end{aligned}
\end{equation*}
Unlike the standard Lagrangian, this augmented formulation incorporates a quadratic penalty. Through the association of a multiplier function $\lambda: \mathcal{X} \rightarrow \mathbb{R}$ with these equality constraints, we formulate the augmented Lagrangian function as follows:
\begin{equation}
\begin{aligned}
L(\pi,p,\lambda,\rho)&=-J(\pi)+\int\lambda(x)(F^{\pi}(x)+p(x))\,dx\\
&+\frac{\rho}{2}\int (F^{\pi}(x)+p(x))^2\,dx.
\end{aligned}
\end{equation}
where the integration is performed over the state space $\mathcal{X}$ with respect to the distribution $\mu$, and $\rho > 0$ denotes the penalty factor. Similar to the standard Lagrangian method, we construct the primal-dual formulation with a dynamic penalty $\rho$, which translates into the saddle-point problem:
\begin{equation*}
    (\pi,p,\lambda)=\max_{\lambda}\min_{\pi,p}L(\pi,p,\lambda,\rho).
\end{equation*}
To obtain the optimal policy, ALaM iteratively advances the overall training process through three coupled mechanisms: optimization of the primal variables (the policy and the slack function), update for the multiplier function, and an adaptive adjustment scheme for the penalty factor.

\subsection{Primal variables update}
At the $k$-th iteration, given the current multiplier $\lambda^k$ and penalty factor $\rho^k$, we update the primal variables by solving the following joint minimization problem:
\begin{equation}\label{eq:4}
    \left(\pi^{k+1},p^{k+1}\right)=\argmin_{\pi,p\ge0} L\left(\pi,p,\lambda^k,\rho^k\right).
\end{equation}
Since the objective function is quadratic with respect to the slack variable $p(x)$, we can analytically eliminate it, reducing the formulation \eqref{eq:4} to an optimization problem depending only on $\pi$. That is, for any fixed state $x$, the subproblem for $p(x)$ is given by:
\begin{equation*}
    p^{k+1}(x)=\argmin_{p(x) \ge 0} \lambda^k(x)(F^{\pi}(x)+p(x))+\frac{\rho^k}{2}(F^{\pi}(x)+p(x))^2,
\end{equation*}
which admits a closed-form solution:
\begin{equation*}
     p^{k+1}(x)=\max\left\{-\frac{\lambda^k(x)}{\rho^k}-F^{\pi}(x),0\right\}.
\end{equation*}
Substituting the expression back into $L(\pi, p,\lambda^k, \rho^k)$ produces a simplified augmented Lagrangian:
\begin{equation*}
\begin{aligned}
     L(\pi,\lambda^k,\rho^k)&=-J(\pi)+\int\lambda^k(x)\max\left\{-\frac{\lambda^k(x)}{\rho^k},F^{\pi}(x)\right\}\,dx\\
     &+\frac{\rho^k}{2}\int \max\left\{-\frac{\lambda^k(x)}{\rho^k},F^{\pi}(x)\right\}^2\,dx.
\end{aligned}
\end{equation*}
Consequently, the policy update rule reduces to the subproblem:
\begin{equation}\label{eq:5}
    \pi^{k+1}=\argmin_{\pi} L\left(\pi,\lambda^k,\rho^k\right).
\end{equation}

\subsection{Multiplier update}
We initiate the multiplier update through the derivation of the theoretical update rule within the functional space. An optimal solution for the original problem must fulfill the Karush-Kuhn-Tucker (KKT) stationarity condition:
\begin{equation}\label{eq:6}
    -\partial J(\pi)+\int \lambda(x)\partial F^\pi(x)\,dx=0, 
\end{equation}
where $\partial$ denotes the derivative with respect to $\pi$. Concurrently, at the $k$-th iteration, the newly obtained optimal solutions $\pi^{k+1}$ and $p^{k+1}$ from the augmented Lagrangian subproblem \eqref{eq:4} satisfy their stationarity optimality condition:
\begin{equation}\label{eq:7}
\begin{aligned}
    & -\partial J(\pi^{k+1})+\\
    & \int \left(\lambda^k(x)+\rho^k\left(F^{\pi^{k+1}}(x)+p^{k+1}(x)\right)\right)\partial F^{\pi^{k+1}}(x)\,dx=0.
\end{aligned}
\end{equation}
To align the augmented condition \eqref{eq:7} with the KKT condition \eqref{eq:6}, the multiplier function must be updated to match the effective multiplier term $\lambda^k + \rho^k(F^{\pi^{k+1}} + p^{k+1})$. Substituting the derived closed-form solution for the optimal slack variable $p^{k+1}$, we obtain the dual target:
\begin{equation*}
    \lambda_{\text{target}}^{k+1}(x) = \max \left\{ \lambda^k(x)+\rho^kF^{\pi^{k+1}}(x),0\right\},\quad \forall x \in \text{supp}(\mu).
\end{equation*}
Because the state space $\mathcal{X}$ contains an infinite continuum of constraints, a discrete point-wise update for every multiplier is computationally intractable. This necessitates the representation of the multiplier function via a neural network. Through this parameterization, we reformulate the dual update as a supervised regression task. Specifically, the network is trained to match the theoretical target by minimizing the mean squared error objective:
\begin{equation}\label{eq:8}
    L_\lambda=\int\left(\lambda(x)-\lambda_{\text{target}}^{k+1}(x)\right)^2\,dx.
\end{equation}

\subsection{Penalty factor update}
Unlike traditional multiplier methods that alternate between policies and multipliers, ALaM dynamically adjusts the penalty factor to provide instantaneous feedback on constraint violations. This adaptive implementation serves as a crucial compensation for delayed dual response, and thereby suppresses algorithmic oscillations. To illustrate this mechanism, consider the scenario where constraints are active. Under such conditions, the closed-form update of the dual variable simplifies to:
\begin{equation*}
    \lambda^{k+1}(x) = \lambda^k(x)+\rho^kF^{\pi^{k+1}}(x),
\end{equation*}
which implies
\begin{equation*}
    F^{\pi^{k+1}}(x) =\frac{1}{\rho^k}\left(\lambda^{k+1}(x)-\lambda^k(x)\right).
\end{equation*}
Therefore, the constraint violation at any state $x$ scales inversely with the penalty factor $\rho^k$. This indicates that we can reduce infeasibility by increasing the penalty parameter. To quantify the overall constraint violation, we monitor the following metric at each iteration:
\begin{equation*}
\begin{aligned}
    v^k&=\sqrt{\int\left(p^{k+1}(x)+F^{\pi^{k+1}}(x)\right)^2\,dx}\\
    &=\sqrt{\int\max\left\{F^{\pi^{k+1}}(x),-\frac{\lambda^{k}(x)}{\rho^{k}}\right\}^2\,dx}.
\end{aligned}
\end{equation*}
If $v^k > 1/\rho^k$, we amplify the penalty factor via $\rho^{k+1} \leftarrow \sigma \rho^k$, where $\sigma > 1$ is a scaling coefficient. Classical augmented Lagrangian theory \cite{bertsekas2014constrained} shows that this method possesses exact penalty properties. Once $\rho$ exceeds a finite threshold, the optimization of the augmented objective recovers the optimal solution of the original constrained problem. We utilize this property to improve the penalty factor update scheme. Specifically, we restrict the penalty parameter to an upper bound $\rho_{\max}$ through the assignment $\rho^{k+1} \leftarrow \min\{\sigma \rho^k, \rho_{\max}\}$. This clipping mechanism safely prevents gradient explosion during training while effectively guaranteeing constraint satisfaction.

\section{Convergence analysis}
In this section, we establish the theoretical convergence guarantees for ALaM. We assume that the state space $\mathcal{X}$ is compact, and the action space $\mathcal{A}$ is a compact and convex set. These assumptions are well-justified by practical continuous control scenarios, such as finite sensor ranges and actuator saturation. The Lagrangian $L(\pi, \lambda)$ and the dual functional $d(\lambda)$ are defined as:
\begin{equation*}
    L(\pi,\lambda)=-J(\pi)+\int\lambda(x)F^{\pi}(x)\,dx,\quad d(\lambda)=\inf_{\pi}L(\pi,\lambda).
\end{equation*}

For a rigorous analysis, we formulate the dual problem in a Hilbert space. Let $H = L^2(\mathcal{X})$ be the Hilbert space of square-integrable functions defined on $\mathcal{X}$, with its closed convex positive cone denoted by $H_+ = \{ \lambda \in H \mid \lambda(x) \ge 0, \, \forall x \in \mathcal{X} \}$. Equipped with the standard inner product $\langle f, g \rangle = \int_{\mathcal{X}} f(x)g(x) \, dx$ and the induced norm $\|f\|^2 = \langle f, f \rangle$, the standard and augmented Lagrangian can be rewritten as:
\begin{equation*}
    L(\pi,\lambda)=-J(\pi)+\left \langle\lambda, F^\pi \right \rangle,
\end{equation*}
and 
\begin{align*}
     L(\pi,\lambda,\rho)&=-J(\pi)+ \left\langle\lambda, \max\left\{-\frac{\lambda}{\rho}, F^\pi\right\}  \right\rangle\\
     & +\frac{\rho}{2}\left\|\max\left\{ -\frac{\lambda}{\rho}, F^\pi\right\}\right\|^2\\
     &=-J(\pi)+\frac{1}{2\rho}\left\|\max\left\{ 0, \lambda+\rho F^\pi\right\}\right\|^2-\frac{1}{2\rho}\left\|\lambda\right\|^2.
\end{align*}
Throughout the analysis, we assume that the multiplier sequence remains square-integrable, i.e., $\{ \lambda^k \} \subset H$. This condition naturally holds in practice: since the multiplier is parameterized by a neural network over the compact space $\mathcal{X}$, it is inherently bounded and thus square-integrable. Furthermore, because neural networks possess sufficient expressive power, we develop our theoretical analysis under the assumption of exact function approximation. The update rules are summarized as follows:
\begin{equation*}
\left\{\begin{aligned}
\pi^{k+1} &= \argmin_{\pi} L(\pi,\lambda^k, \rho^k),\\
\lambda^{k+1}(x) &= \max\left\{0, \lambda^k(x)+\rho^k F^{\pi^{k+1}}(x)\right\},\quad \forall x,\\[8pt]
\rho^{k+1} &= \min\{\sigma \rho^k, \rho_{\max}\}, \quad \text{if } v^k>1/\rho^k.
\end{aligned}\right.
\end{equation*}
Since the sequence $\{ \rho^k \}$ geometrically increases and is bounded by $\rho_{\max}$, it remains constant at a terminal value $\bar{\rho}$ after finite iterations. Our subsequent convergence analysis relies on the following assumptions.

\begin{assumption}
The convergence requires the following conditions: 
\begin{enumerate}
    \item $J(\pi)$ and the mapping $\pi \mapsto F^\pi$ are bounded over the policy space.

    \item  $J(\pi)$ and $\pi \mapsto F^\pi$ are strongly continuous with respect to $\pi$.
    
    \item $J(\pi)$ is concave, and $F^\pi(x)$ is convex with respect to $\pi$ for all $x \in \mathcal{X}$.
\end{enumerate}
\end{assumption}

Several remarks regarding these assumptions are in order. Assumption 1 naturally holds in physical systems, where hardware and energy limits bound the returns and costs. Assumption 2 imposes standard continuity to ensure convergence under the weak topology \cite[Definition 26.1]{zeidler2013nonlinear}. Finally, while the literature typically achieves convexity by mapping RL problems into the occupancy measure space \cite{paternain2019constrained}, extending this transformation to our setting poses significant analytical challenges due to the infinite-dimensional dual variable and the state-wise constraints. Thus, Assumption 3 serves as a necessary condition to guarantee the existence of a global saddle point.

Since the action space $\mathcal{A}$ is compact, any deterministic policy $\pi: \mathcal{X} \to \mathcal{A}$ is bounded. Combined with the fact that the compact state space $\mathcal{X}$ possesses a finite measure, this ensures that the policy space $\Pi$ is uniformly bounded within the Hilbert space $L^2(\mathcal{X})$. We now consider the Lagrange dual problem:
\begin{equation*}
    \max_{\lambda \in H_+}d(\lambda),
\end{equation*}
which is equivalent to solving the unconstrained minimization problem:
\begin{equation}\label{eq:D}
    \min_\lambda\Phi(\lambda),
\end{equation}
where $\Phi(\lambda)=-d(\lambda)+\delta_{H_+}(\lambda)$. Here, $\delta_{H_+}$ denotes the indicator function of $H_+$, mapping to $0$ if $\lambda \in H_+$ and $\infty$ otherwise. Finally, let $T = \partial \Phi$ denote the subdifferential operator of $\Phi$.

\begin{lemma}\label{lemma:1}
$T $ is a maximal monotone operator. Specifically, an operator $T$ is monotone if $\langle y_1 - y_2, \lambda_1 - \lambda_2 \rangle \ge 0$ for any $y_1 \in T(\lambda_1)$ and $y_2 \in T(\lambda_2)$, and it is maximal in that there exists no other monotone operator $T'$ satisfying $\text{Graph}(T) \subsetneq \text{Graph}(T')$, where $\text{Graph}(T)=\{(\lambda, y)\in H \times H: y \in T(\lambda)\}$.
\end{lemma}

\begin{proof}
By Rockafellar's theorem \cite[Theorem A]{rockafellar1970maximal}, the subdifferential of a proper, lower semi-continuous (l.s.c.), and convex functional is guaranteed to be a maximal monotone operator. We verify these properties for $\Phi$. Since $H_+$ is a closed convex set, its indicator functional $\delta_{H_+}$ is convex and l.s.c. The negative dual functional $-d(\lambda) = -\inf_\pi L(\pi, \lambda) = \sup_\pi \{-L(\pi, \lambda)\}$ is the point-wise supremum of a family of affine functions in $\lambda$. Such a supremum is convex and l.s.c. Their sum $\Phi = -d + \delta_{H_+}$ preserves both properties. By Assumption 1, since $J(\pi)$ and $F^\pi$ are bounded over $\Pi$, $\Phi(\lambda)$ is proper. Therefore, we conclude that $T = \partial \Phi$ is maximal monotone.
\end{proof}

\begin{lemma} \label{lemma:2}
    For any index $k\ge 1$, we have $-F^{\pi^k} \in \partial(-d)(\lambda^k)$. 
\end{lemma}
\begin{proof}
For $k\ge 0$, the update rule $\pi^{k+1}=\argmin_{\pi} L(\pi, \lambda^k,\rho^k)$ implies the following first-order stationarity condition:
\begin{equation}\label{eq:9}
    \begin{aligned}
    0&\in \partial_\pi L\left(\pi^{k+1},\lambda^k, \rho^k\right)\\
    &=\left.\partial_\pi \left(-J(\pi)+\frac{1}{2\rho^k}\left\|\max\left\{0,\lambda^k+\rho^kF^{\pi}\right\}\right\|^2\right)\right|_{\pi=\pi^{k+1}}\\
    &=\partial_\pi\left.\left(- J(\pi)+\left\langle \max\left\{0, \lambda^k+\rho^kF^{\pi^{k+1}}\right\}, F^\pi\right \rangle\right) \right|_{\pi=\pi^{k+1}}\\
    &=-\partial_\pi J\left(\pi^{k+1}\right)+\partial_\pi\left\langle \lambda^{k+1}, F^{\pi^{k+1}}\right\rangle\\
    &=\partial_\pi L\left(\pi^{k+1},\lambda^{k+1}\right).
\end{aligned}
\end{equation}
By Assumption 2, both $-J(\pi)$ and $F^\pi$ are convex in $\pi$, $L(\pi,\lambda)$ is also convex. Consequently, Eq. \eqref{eq:9} ensures that $\pi^{k+1} \in \argmin_{\pi} L(\pi, \lambda^{k+1})$, evaluating the dual function as:
\begin{equation}\label{eq:10}
    d\left(\lambda^{k+1}\right)=L\left(\pi^{k+1},\lambda^{k+1}\right)=-J\left(\pi^{k+1}\right)+\big\langle \lambda^{k+1}, F^{\pi^{k+1}}\big\rangle.
\end{equation}
For any $\lambda \in H_+$, the definition of the dual function provides:
\begin{equation}\label{eq:11}
    d(\lambda)=\inf_\pi L(\pi,\lambda)\le -J\big(\pi^{k+1}\big)+\big\langle \lambda,  F^{\pi^{k+1}}\big\rangle,
\end{equation}
substituting Eq. \eqref{eq:10} into Eq. \eqref{eq:11}, we obtain
\begin{equation*}
    -d(\lambda)\ge -d\left(\lambda^{k+1}\right)+\big\langle-F^{\pi^{k+1}}, \lambda-\lambda^{k+1}\big\rangle.
\end{equation*}
By the definition of the subdifferential, this implies $-F^{\pi^{k+1}} \in \partial (-d)(\lambda^{k+1})$, which completes the proof.
\end{proof}

We are now positioned to establish the convergence of ALaM. Since the dual problem \eqref{eq:D} is convex, a multiplier $\lambda^\star$ is optimal if and only if $0 \in \partial\Phi(\lambda^\star)$. Let $\Omega^\star = \{\lambda \in H : 0 \in \partial\Phi(\lambda)\}$ denote the set of optimal dual solutions.
The following theorem demonstrates that the sequence of multiplier converges weakly to a point in $\Omega^\star$.

\begin{theorem}\label{thm:1}
    The sequence $\{\lambda^k\}$ converges weakly to an optimal dual solution $\lambda^\star \in \Omega^\star$.
\end{theorem}
\begin{proof}
The dual update rule $\lambda^{k+1}=\max\big\{0,\lambda^k+\rho^kF^{\pi^{k+1}}\big\}$ is equivalent to the projection $ \lambda^{k+1}=P_{H_+}\big(\lambda^k+\rho^kF^{\pi^{k+1}}\big)$. Let $N_{H_+}(\lambda)=\{y\in H:\langle y, \lambda'-\lambda\rangle \le 0,\; \forall \lambda' \in H_+\}$ be the normal cone of $H_+$ at $\lambda$. The projection implies
\begin{equation*}
    \lambda^k+\rho^kF^{\pi^{k+1}}-\lambda^{k+1} \in N_{H_+}(\lambda^{k+1}).
\end{equation*}
Since the subdifferential $\partial\delta_{H_+}$ coincides with the normal cone $N_{H_+}$ for convex indicator functions, it follows that:
\begin{equation*}
    \frac{\lambda^k-\lambda^{k+1}}{\rho^k} \in -F^{\pi^{k+1}}+\partial\delta_{H_+}(\lambda^{k+1}).
\end{equation*}
Invoking Lemma \ref{lemma:2}, we have 
\begin{equation*}
    \frac{\lambda^k-\lambda^{k+1}}{\rho^k} \in \partial(-d)(\lambda^{k+1})+\partial\delta_{H_+}(\lambda^{k+1})=T(\lambda^{k+1}),
\end{equation*}
which can be rearranged as $ \lambda^k \in (I+\rho^kT)(\lambda^{k+1})$. Since $T$ is maximal monotone, Minty theorem \cite[Theorem 21.1, Proposition 23.7]{bauschke2017convex} ensures that the resolvent $J^k=(I+\rho^k T)^{-1}$ is everywhere defined and single-valued. Thus, the update corresponds to $\lambda^{k+1} =J^k \lambda^k$.

Let $\hat{\lambda}\in \Omega^\star$ be an arbitrary optimal solution, meaning $0 \in T(\hat{\lambda})$. By the monotonicity of $T$, we have
\begin{equation*}
    \left\langle \frac{\lambda^k-\lambda^{k+1}}{\rho^k}-0, \lambda^{k+1}-
    \hat{\lambda}\right\rangle \ge 0,
\end{equation*}
which implies $\langle \lambda^k-\lambda^{k+1}, \lambda^{k+1}-
    \hat{\lambda}\rangle \ge 0$. Thus, expanding the squared norm yields:
\begin{equation}\label{eq:14}
\begin{aligned}
    \|\lambda^k-\hat{\lambda}\|^2&=\|\lambda^k-\lambda^{k+1}+\lambda^{k+1}-\hat{\lambda}\|^2\\
    &=\|\lambda^k-\lambda^{k+1}\|^2+\|\lambda^{k+1}-\hat{\lambda}\|^2\\
    &+2\langle \lambda^k-\lambda^{k+1}, \lambda^{k+1}-\hat{\lambda}\rangle\\
    &\ge \|\lambda^k-\lambda^{k+1}\|^2+\|\lambda^{k+1}-\hat{\lambda}\|^2.
\end{aligned}
\end{equation}
Rearranging the terms, we obtain $\|\lambda^{k+1}-\hat{\lambda}\|^2\le\|\lambda^k-\hat{\lambda}\|^2-\|\lambda^k-\lambda^{k+1}\|^2$.
This establishes that the sequence $\{\|\lambda^k-\hat{\lambda}\|^2\}$ is monotonically decreasing. Being bounded below by zero, it must converge, which in turn guarantees that $\{\lambda^k\}$ is bounded. Taking the sum of Eq. \eqref{eq:14} over $k$, we have
\begin{equation*}
    \sum_{k=0}^{\infty}\|\lambda^k-\lambda^{k+1}\|^2\le \|\lambda^0-\hat{\lambda}\|^2 < \infty,
\end{equation*}
which necessitates that $\lim_{k \to \infty}\|\lambda^k-\lambda^{k+1}\|^2=0$. Since $\{\lambda^k\}$ is bounded in the reflexive Hilbert space $L^2(\mathcal{X})$, the Eberlein-\v{S}mulian theorem \cite[Theorem 3.18]{brezis2011functional} guarantees the existence of a weakly convergent subsequence $\lambda^{k_j} \rightharpoonup \bar{\lambda}$. Moreover, since $\|\lambda^{k_j}-\lambda^{k_j+1}\| \to 0$, it follows that $\lambda^{k_j+1}\rightharpoonup \bar{\lambda}$ as well. Let $y^{k+1}=\frac{\lambda^k-\lambda^{k+1}}{\rho^k}$. With the penalty parameter $\rho^k$ bounded away from zero, the convergence $\|\lambda^k-\lambda^{k+1}\|\to 0$ ensures that $y^{k_j+1} \rightarrow 0$.
 
 We now assert that $0\in T(\bar{\lambda})$. Since $T$ is maximal monotone, any $(\lambda, y) \in \text{Graph}(T)$ satisfies
 \begin{equation*}
     \langle y-y^{k_j+1},\lambda-\lambda^{k_j+1} \rangle \ge 0.
 \end{equation*}
 Because $y^{k_j+1} \to 0$ and $\lambda^{k_j+1} \rightharpoonup \bar{\lambda}$, the inner product is continuous with respect to this mixed convergence. Taking the limit as $j\rightarrow\infty$ yields:
 \begin{equation}\label{eq:16}
 \langle y,\lambda-\bar{\lambda} \rangle\ge 0.
 \end{equation}
 If $0\notin T(\bar{\lambda})$, we could construct a new operator $T'$ with $\text{Graph}(T')=\text{Graph}(T)\cup\{(\bar{\lambda},0)\}$. From Eq. \eqref{eq:16}, $T'$ is also monotone and $\text{Graph}(T) \subsetneq \text{Graph}(T')$, which is a contradiction with the maximality of $T$. Therefore, $0\in T(\bar{\lambda})$, meaning $\bar{\lambda} \in \Omega^\star$.

 In summary, for any $\hat{\lambda} \in \Omega^\star$, $\{\|\lambda^k-\hat{\lambda}\|\}$ is monotonically decreasing with a lower bound, and every weak sequential cluster point of  $\{\lambda^k\}$ belongs to $\Omega^\star$. By Opial lemma \cite[Theorem 5.5]{bauschke2020correction}, $\{\lambda^k\}$ weakly converges to a point in $\Omega^\star$. This concludes the proof.
\end{proof}

The next theorem establishes the convergence of the policy sequence $\{\pi^k\}$ and guarantees that any of its weak sequential cluster points is optimal for the original constrained problem.
\begin{theorem}
 Any weak sequential cluster point of the policy sequence $\{\pi^{k}\}$ is a globally optimal policy of problem \eqref{eq:2}.
\end{theorem}
\begin{proof}
Because the policy space $\Pi$ is a bounded subset of the reflexive Hilbert space $L^2(\mathcal{X})$, the Eberlein-\v{S}mulian theorem guarantees the existence of a weakly convergent subsequence $\pi^{k_j} \rightharpoonup \bar{\pi}$.

We first verify $\bar{\pi} \in \Pi$ by proving that the policy space $\Pi= \{ \pi \in L^2(\mathcal{X}) \mid \pi(x) \in \mathcal{A} \text{ a.e.}\}$ is weakly closed. Since the action space $\mathcal{A}$ is compact and convex, $\Pi$ is inherently convex. To establish strong closedness, consider any sequence $\{\pi_n\} \subset \Pi$ satisfying $\|\pi_n - \pi\| \to 0$. By the Riesz theorem \cite[Theorem 4.9]{zeidler2013nonlinear}, there exists a subsequence converging to $\pi$ a.e. Since $\mathcal{A}$ is a closed set, this pointwise limit must also reside within $\mathcal{A}$ a.e., ensuring $\pi \in \Pi$. Thus, $\Pi$ is strongly closed. By Mazur's lemma \cite[Theorem 3.7]{zeidler2013nonlinear}, a strongly closed and convex subset of a Hilbert space is inherently weakly closed. Consequently, the weak convergence $\pi^{k_j} \rightharpoonup \bar{\pi}$ guarantees $\bar{\pi} \in \Pi$.

Next, we evaluate the KKT conditions for the limit point. Given the strong continuity of $F$ and the limit $\|\lambda^{k_j-1}-\lambda^{k_j}\| \rightarrow0$ established in Theorem \ref{thm:1}, it follows that  $\frac{\lambda^{k_j-1}-\lambda^{k_j}}{\rho^{k_j-1}}+F^{\pi^{k_j}}\rightarrow F^{\bar{\pi}}$. Recall that this sequence belongs to $N_{H_+}(\lambda^{k_j})$ and $N_{H_+}(\lambda)$ is a maximal monotone operator. Applying a maximal monotonicity argument identical to that in Theorem \ref{thm:1}, we obtain $F^{\bar{\pi}} \in N_{H_+}(\lambda^\star)$. 

Since $N_{H_+}(\lambda^\star)$ is a cone, it implies $\bar{\rho}F^{\bar{\pi}} \in N_{H_+}(\lambda^\star)$. This is equivalent to $ \lambda^\star+\bar{\rho}F^{\bar{\pi}} -\lambda^\star\in N_{H_+}(\lambda^\star)$, leading to the projection
\begin{equation*}
    \lambda^\star=P_{H_+}(\lambda^\star+\bar{\rho}F^{\bar{\pi}})=\max\{0,\lambda^\star+\bar{\rho}F^{\bar{\pi}}\}.
\end{equation*}
Therefore, for any $x\in \mathcal{X}$, if $\lambda^\star(x)>0$, it necessitates $\lambda^\star(x)=\lambda^\star(x)+\bar{\rho}F^{\bar{\pi}}(x)$, forcing $F^{\bar{\pi}}(x)=0$; conversely, if $\lambda^\star(x)=0$, it requires $\lambda^\star(x)+\bar{\rho}F^{\bar{\pi}}(x)\le0 $, ensuring $F^{\bar{\pi}}(x)\le0$. Thus, the primal feasibility $F^{\bar{\pi}}(x)\le0$, the dual feasibility $\lambda^\star(x)\ge0$, and the complementary slackness $\lambda^\star(x)F^{\bar{\pi}}(x)=0$ hold. 

Finally, we establish the optimality condition for $\bar{\pi}$. At any iteration $k_j$, we have $L(\pi^{k_j}, \lambda^{k_j-1}, \rho^{k_j-1}) \le L(\pi, \lambda^{k_j-1}, \rho^{k_j-1})$ for all $\pi \in \Pi$. Because $-J(\pi)$ and $F^\pi$ are convex and strongly continuous,  $L(\pi, \lambda, \rho)$ is convex and strongly continuous, and thus l.s.c. with respect to $\pi$ in the weak topology \cite[Corollary 3.9]{zeidler2013nonlinear}. Since the inner product operator is continuous under weak convergence, together with $F^{\pi^{k_j}} \to F^{\bar{\pi}}$ and $\lambda^{k} \rightharpoonup \lambda^{\star} \in \Omega^\star$, $L(\pi, \lambda, \rho)$ is l.s.c. with respect to the joint sequence. Therefore,
\begin{equation*}
\begin{aligned}
 &L(\bar{\pi}, \lambda^\star, \bar{\rho})\le \liminf_{j \to \infty} L(\pi^{k_j}, \lambda^{k_j-1}, \rho^{k_j-1}) \\
 \le& \lim_{j \to \infty} L(\pi, \lambda^{k_j-1}, \rho^{k_j-1}) = L(\pi, \lambda^\star, \bar{\rho}).   
\end{aligned}
\end{equation*}
This implies $\bar{\pi} \in \argmin_{\pi \in \Pi} L(\pi, \lambda^\star, \bar{\rho})$, i.e.,
\begin{equation*}
     0 \in \partial_\pi L(\bar{\pi}, \lambda^\star, \bar{\rho})
     =\partial_\pi L(\bar{\pi},\lambda^\star).
\end{equation*}

Overall, the pair $(\bar{\pi}, \lambda^\star)$ fulfills all KKT conditions. Given the convexity of problem \eqref{eq:2}, these conditions are sufficient, ensuring $\bar{\pi}$ is indeed a globally optimal feasible policy. 
\end{proof}

\section{Practical implementations}
In this section, we present the practical implementation of ALaM in the context of deep RL. Taking advantage of the enhanced policy exploration and training robustness offered by soft actor-critic (SAC) \cite{haarnoja2018soft}, we incorporate it as the foundational algorithm for reward maximization within the ALaM framework; the resulting algorithm is denoted as SAC-ALaM. 
We set the feasibility function as the cost value function. The architecture involves training several neural networks (with subscripts denoting their learnable parameters): a policy network $\pi_\theta$, two Q-networks $Q_{\phi_1}$, $Q_{\phi_2}$, two cost Q-networks $Q^c_{\psi_1}$, $Q^c_{\psi_2}$, and a multiplier network $\lambda_w$. Correspondingly, we maintain four target Q-networks parametrized by $\bar{\phi}_1, \bar{\phi}_2, \bar{\psi}_1$, and $\bar{\psi}_2$. 

To update the policy, we aim to solve the subproblem:
\begin{equation*}
\begin{aligned}
 \theta^{k+1}&=\argmin_{\theta}-J(\pi_{\theta})+\int\lambda_{w}(x)\max\left\{-\frac{\lambda_{w}(x)}{\rho},F_{\psi}(x)\right\}\,dx\\
    &+\frac{\rho}{2}\int \max\left\{-\frac{\lambda_{w}(x)}{\rho},F_{\psi}(x)\right\}^2\,dx,   
\end{aligned}
\end{equation*}
where $F_{\psi}(x)$ is the state-wise cost value of policy $\pi_{\theta}$. In the SAC framework, we replace the standard objective $J(\pi)$ with the soft Q-function, and estimate the integral using empirical expectations over sampled transitions. Combining these components, the policy loss is constructed as:
\begin{equation}
\begin{split}
    L_{\pi}(\theta)=
    &\mathbb{E}_{x \sim \mathcal{B},a\sim \pi_\theta}\left [\alpha \log\pi_\theta(a|x)-\min_{i\in \{1,2\}}Q_{\phi_i}(x,a)\right ]\\
    +&\mathbb{E}_{x \sim \mathcal{B}}\left [\lambda_{w}(x)\max \left\{-\frac{\lambda_{w}(x)}{\rho},\hat{F}_{\psi}(x)\right\}\right]\\
    +&\mathbb{E}_{x \sim \mathcal{B}}\left[\frac{\rho}{2}\max\left\{-\frac{\lambda_{w}(x)}{\rho},\hat{F}_{\psi}(x)\right\}^2
    \right],
\end{split}
\end{equation}
where $\alpha$ denotes the temperature. To prevent risk underestimation, the cost value estimation relies on the clipped double Q-learning trick \cite{fujimoto2018addressing}. In practice, we compute $\hat{F}_{\psi}(x)$ using Monte Carlo sampling with $K$ actions drawn from the current policy:
\begin{equation*}
    \hat{F}_{\psi}(x)=\frac{1}{K}\sum_{k=1}^{K} \max_{i \in \{1,2\}} Q^c_{\psi_i}(x,a_k), \quad \text{where} \ a_k \sim \pi_\theta(\cdot|x) .
\end{equation*}

The loss functions for the critic networks are the same as those in SAC:
\begin{equation*}
\begin{split}
    L_{Q}(\phi_i)&=\mathbb{E}_{(x, a, r, x') \sim \mathcal{B}}\left[ (y_{\phi}-Q_{\phi_i}(x,a))^2\right],\\[6pt]
    \ y_\phi &=r(x,a)+\gamma\left(\min_{i \in\{1,2\}}Q_{\bar{\phi}_i}(x',a')-\alpha \log\pi_\theta(a'|x')\right),\\[6pt]
    L_{Q^c}(\psi_i)&=\mathbb{E}_{(x, a, c, x') \sim \mathcal{B}}\left[ (y_{\psi}-Q^c_{\psi_i}(x,a))^2\right],\\[6pt]
    \ y_\psi& =c(x)+\gamma\max_{i \in\{1,2\}}Q^c_{\bar{\psi}_i}(x',a'),
\end{split}
\end{equation*}
where the target network parameters ($\bar{\phi}_i, \bar{\psi}_i$) are updated via exponential moving averages. The temperature $\alpha$ is automatically tuned in a similar manner to SAC:
\begin{equation*}
    L(\alpha)=\mathbb{E}_{x\sim \mathcal{B},a \sim \pi_\theta}[\alpha(-\log\pi_\theta(a|x)-\mathcal{H})],
\end{equation*}
where $\mathcal{H}$ is the target entropy.

The objective for updating the multiplier network aligns with Eq. \eqref{eq:8}, which can be approximated by
\begin{equation*}
    L_\lambda(w)=\mathbb{E}_{x\sim \mathcal{B}}[(\lambda_{w}(x)-\text{sg}(\max\{\lambda_w(x)+\rho \hat{F}_{\psi}(x),0\}))^2],
\end{equation*}
where $\text{sg}(\cdot)$ denotes the stop-gradient operator. To ensure a smooth and precise approximation, we perform multiple small-step gradient updates on this objective per cycle. Given that the multiplier updates rely on the policy being near convergence, we employ a two-timescale learning scheme. By introducing a delay interval $m_{\lambda}$, the multiplier network is updated only every $m_{\lambda}$ policy steps. Following each dual update, we evaluate the average constraint violation metric
\begin{equation*}
    v=\sqrt{\mathbb{E}_{x\sim \mathcal{B}}\left [\max \left\{\hat{F}_{\psi}(x),-\frac{\lambda_w(x)}{\rho}\right \}^2\right]}.
\end{equation*}
If $v>\frac{1}{\rho}$, we increase the penalty factor via $\rho \leftarrow \min\{ \sigma \rho, \rho_{\max}\}$; otherwise, $\rho$ remains unchanged. The pseudocode for SAC-ALaM is summarized in Algorithm \ref{alg}.

\begin{algorithm}[!t]
	\caption{SAC-ALaM}\label{alg}
	\begin{algorithmic}[1]
		\State \textbf{Input} $\theta,\phi_{1,2}, \psi_{1,2}, \bar{\phi}_{1,2}, \bar{\psi}_{1,2},w$.
        \For{each iteration}
            \For{each environment step}
            \State Sample $a_t \sim \pi_\theta(\cdot|x_t)$, $x_{t+1}= f(x_t, a_t)$
            \State $\mathcal{B} \leftarrow \mathcal{B} \ \cup \left\{(x_t,a_t,x_{t+1},r_t,c_t)\right\}$ 
            \EndFor
            \For{each gradient step}
            \State $\phi_i \leftarrow \phi_i-\beta_\phi\nabla L_Q(\phi_i), i=1,2$
            \State $\psi_i \leftarrow \psi_i-\beta_\psi\nabla L_{Q^c}(\psi_i), i=1,2$
            \State $\theta \leftarrow \theta -\beta_\theta \nabla L_\pi(\theta)$
            \State $\alpha \leftarrow \alpha - \beta_{\alpha} \nabla L(\alpha)$
            \If {step $\bmod\ m_{\lambda} == 0$}
                \For {each dual gradient step}
                \State $w \leftarrow w-\beta_w\nabla L_{\lambda}(w)$
                \EndFor
                \If {$v >1/\rho$}
                    \State $\rho \leftarrow \min \{\sigma \rho, \rho_{\max}\}$
                \EndIf
            \EndIf
            \State $\bar{\phi}_i \leftarrow \tau \phi_i+(1-\tau)\bar{\phi}_i,i=1,2$
            \State $\bar{\psi}_i \leftarrow \tau \psi_i+(1-\tau)\bar{\psi}_i,i=1,2$
        \EndFor
    \EndFor
    \State \textbf{Output} $\theta,\phi_{1,2}, \psi_{1,2}, \bar{\phi}_{1,2}, \bar{\psi}_{1,2},w$.
\end{algorithmic}
\end{algorithm}

\section{Experiments}
In this section, we empirically evaluate SAC-ALaM to demonstrate its efficacy and competitive advantage in safe RL. Our evaluation highlights three core strengths:

\begin{itemize}
    \item SAC-ALaM achieves high returns while satisfying safety constraints on standard safe RL benchmark, outperforming state-of-the-art baselines.

    \item SAC-ALaM effectively mitigates training oscillations, enabling stable convergence to the optimal feasible policies.

    \item  The multiplier network learns a well-calibrated risk representation, which provides interpretable safety feedback and demonstrates robust generalization in novel situations.
\end{itemize}

As illustrated in Figure \ref{fig:snapshot}, we evaluate ALaM across a diverse suite of environments from the Safety-Gymnasium benchmark \cite{ji2023safety}. We test multiple agents, ranging from Point and Car featuring low-dimensional state spaces for 2D kinematics and steering, to Swimmer, HalfCheetah, and Ant that introduce high-dimensional continuous control, demanding the agent to simultaneously achieve complex joint coordination, dynamic balance, and hazard avoidance. These agents are assigned to perform various tasks: moving rapidly without exceeding speed limits (Velocity), navigating to the green target while avoiding blue obstacles (Goal), moving as fast as possible along the green circular boundary without leaving the yellow bounded regions (Circle), and reaching and pressing orange buttons while evading hazards (Button). Furthermore, we examine two levels of safety difficulty (Level 1 and Level 2) for some scenarios. By combining these different agents, tasks, and constraint difficulties, we conduct a comprehensive evaluation of SAC-ALaM.

\begin{figure*}[htbp]
    \centering
    \begin{minipage}{0.22\textwidth}
        \includegraphics[width=\linewidth]
        {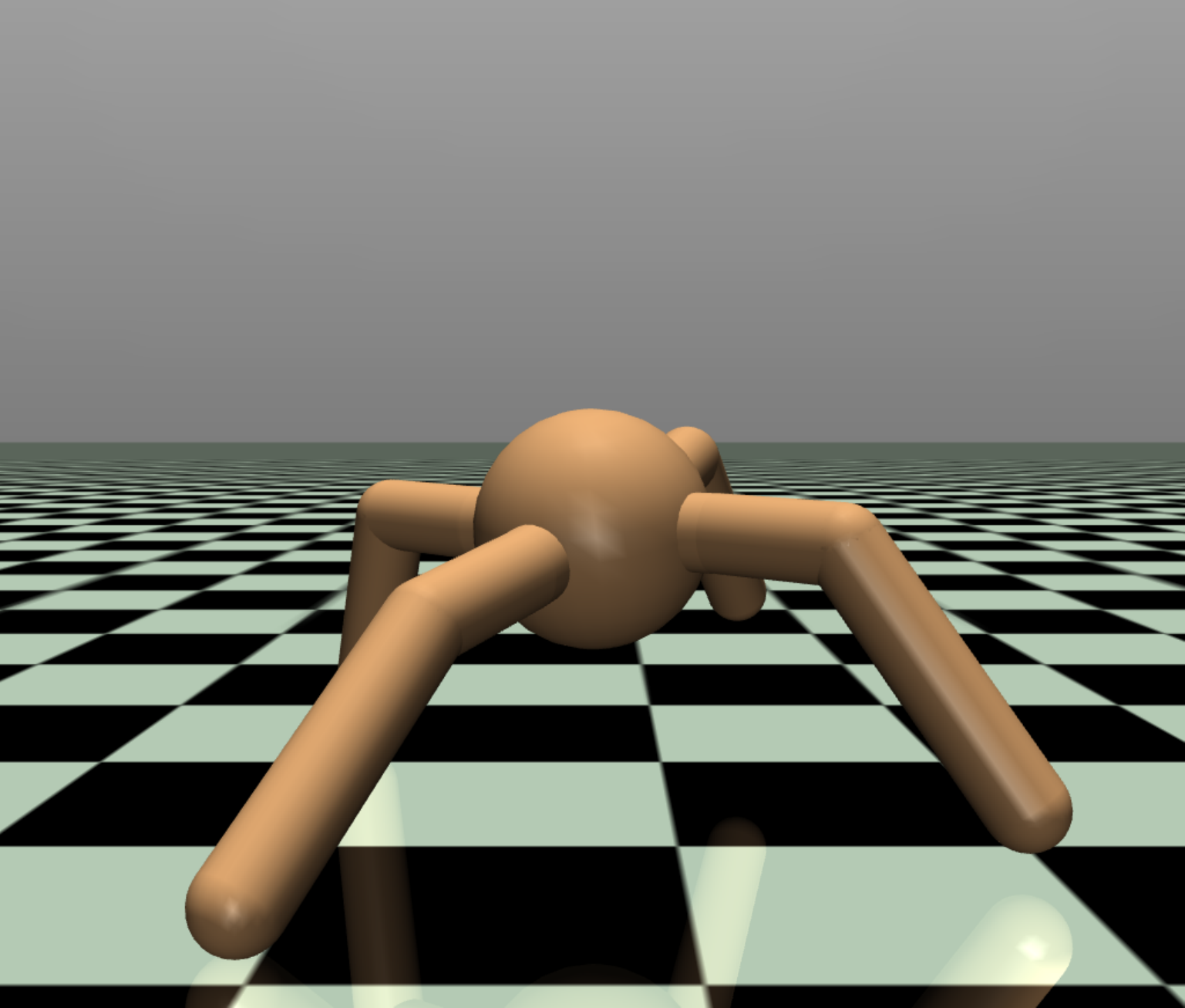}
        \caption*{(a) Velocity}
    \end{minipage}\hfill
    \begin{minipage}{0.22\textwidth}
        \includegraphics[width=\linewidth]
        {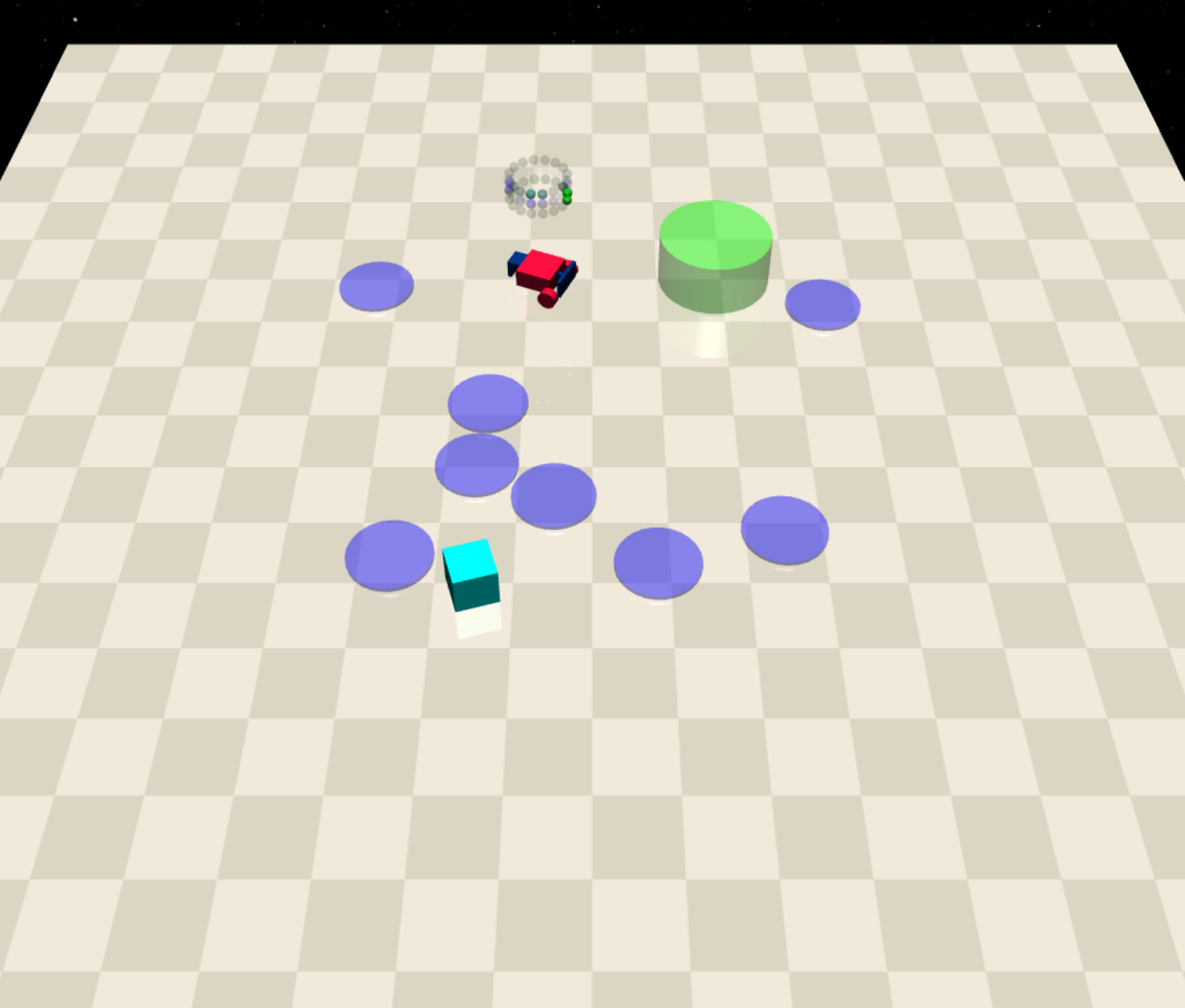}
        \caption*{(b) Goal}
    \end{minipage}\hfill
    \begin{minipage}{0.22\textwidth}
        \includegraphics[width=\linewidth]
        {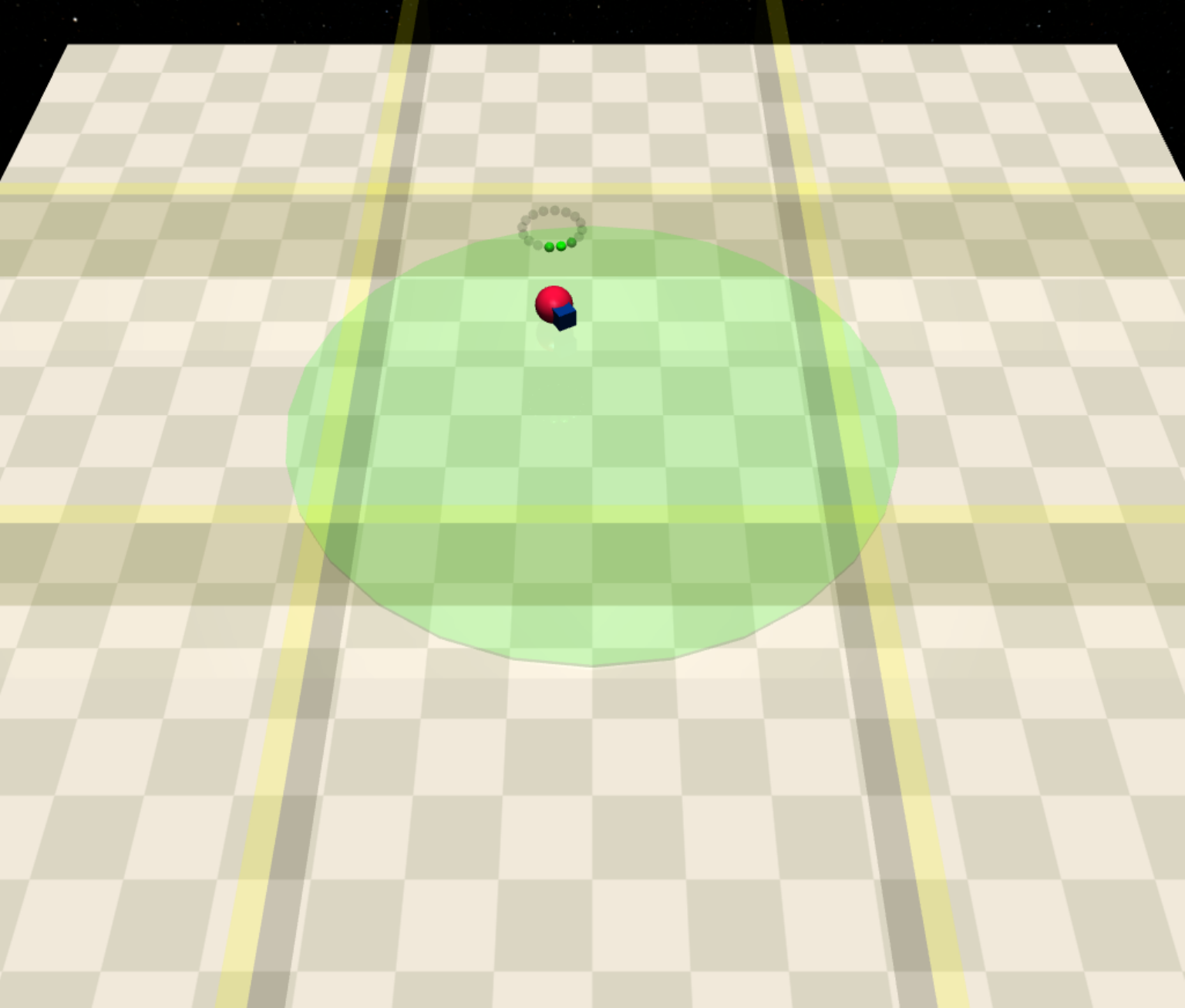}
        \caption*{(c) Circle}
    \end{minipage}\hfill
    \begin{minipage}{0.22\textwidth}
        \includegraphics[width=\linewidth]
        {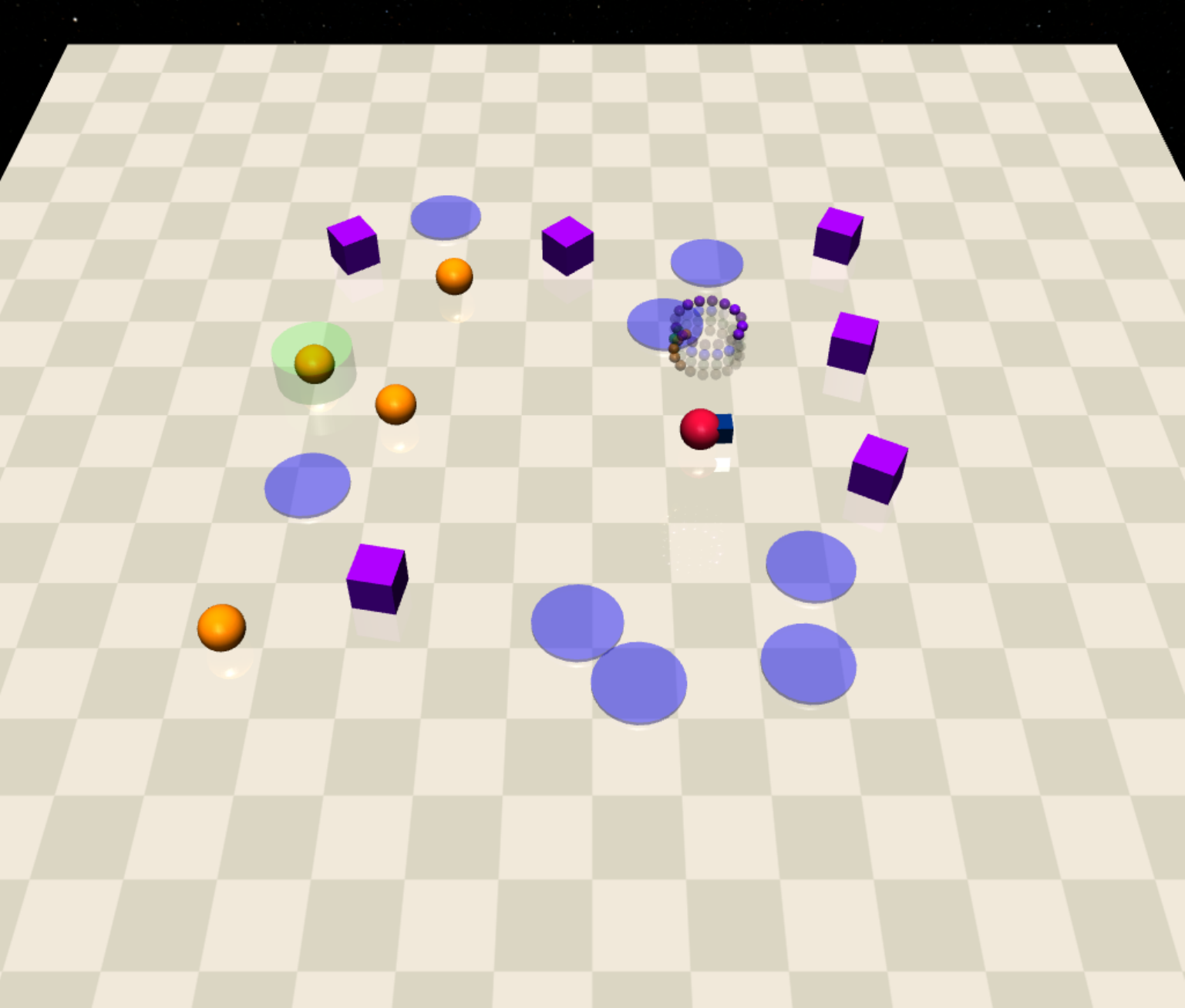}
        \caption*{(d) Button}
    \end{minipage}
\caption{Snapshots of four experimental task scenarios.}
\label{fig:snapshot}
\end{figure*}

\begin{table}[h]
\centering
\caption{Hyperparameters}
\label{tab:hyperparameters}
\renewcommand{\arraystretch}{1.2}
\setlength{\tabcolsep}{18pt} 
\begin{tabular}{lc}
\toprule
Hyperparameter & Value \\
\midrule
Optimizer & RAD \cite{lyu2025conformal}\\
Discount factor ($\gamma$) & 0.99 \\
Hidden layers & 2 \\
Hidden neurons per layer & 256 \\
Activation function & ReLU \\ 
Learning rate & $1 \times 10^{-4}$\\
Learning rate for multiplier & $1 \times 10^{-5}$\\
Total environment steps & $4 \times 10^6$ \\
Replay buffer size  & $2 \times 10^6$ \\
Batch size  & $256$ \\
Target entropy ($\mathcal{H}$) & $-\dim(\mathcal{A})$ \\
Multiplier update interval ($m_{\lambda}$) & 200 \\
Monte Carlo samples ($K$) & 5 \\ 
Dual gradient steps & 5 \\
Maximum penalty ($\rho_{\text{max}}$) & 5.0\\
Penalty scaling factor ($\sigma$) & 1.01\\
\bottomrule
\end{tabular}
\end{table}

We compare SAC-ALaM against a comprehensive set of existing safe RL algorithms. The primary baselines are Lagrangian-based methods, including the scalar multiplier approach (SAC-Lag) \cite{ray2019benchmarking} and the parameterized multiplier network (FAC \cite{ma2021feasible}, denoted as SAC-LagNet for notational consistency). To evaluate stabilization mechanisms, we include the PID Lagrangian method (SAC-PID) \cite{pmlr-v119-stooke20a} and the augmented Lagrangian method (ASAC), adapted from APPO \cite{dai2023augmented} to use SAC. By employing SAC as the shared backbone across all Lagrangian methods, we standardize the return optimization framework, isolating the performance impact of their respective safety mechanisms. Furthermore, we evaluate SAC-FPI \cite{yang2023feasible}, a dynamic programming-based approach applying region-wise policy updates, together with two leading on-policy safe RL algorithms P3O \cite{zhang2022penalized} and CRPO \cite{xu2021crpo}. To prevent overly conservative policies, we relax the strict safety constraint to $F^\pi(x)\le d$, where $d$ is a tunable tolerance threshold with a default value of 0.1. All hyperparameter settings are detailed in Table \ref{tab:hyperparameters}.

\begin{figure}[!t]
\centering
\includegraphics[width=\linewidth]{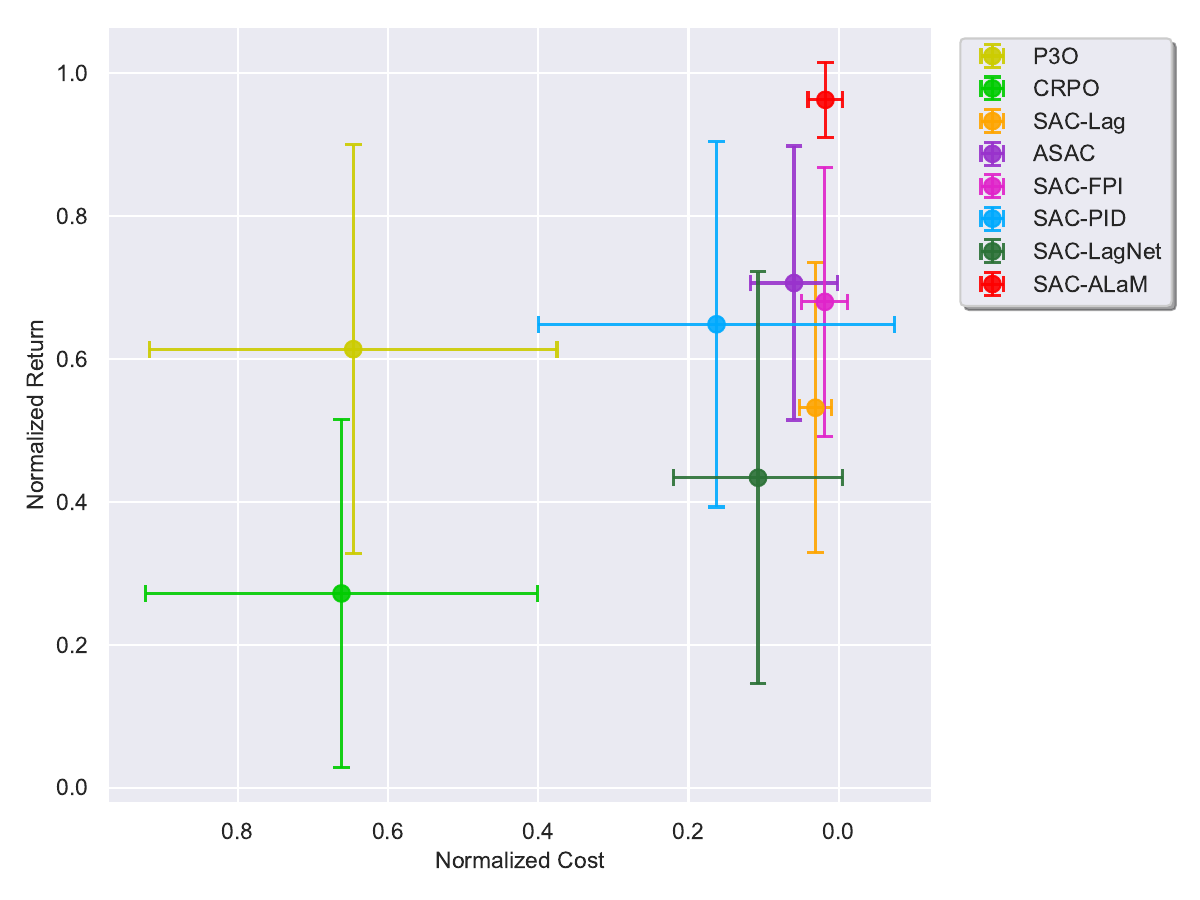}
\caption{Comparison of normalized performance across 8 tasks. The data points and error bars denote the mean and 95\% confidence intervals for episode returns and costs.}
\label{fig:aggregate_performance_profile}
\end{figure}

Figure \ref{fig:training_curves} illustrates the learning curves for episode return and episode cost across the 8 evaluated tasks. SAC-ALaM demonstrates a distinct advantage in balancing reward maximization and constraint satisfaction. In contrast, SAC-Lag and SAC-LagNet suffer from pronounced training oscillations. Although ASAC and SAC-PID partially mitigate these fluctuations, this stability often comes at the expense of final returns. Furthermore, the on-policy baselines struggle with constraint satisfaction, even in relatively simple tasks such as PointGoal1 and CarGoal1. As environmental complexity increases (e.g., in Button environments), the baselines either suffer from severe constraint violations (SAC-Lag, P3O, CRPO) or exhibit excessive conservatism, thereby compromising final performance (SAC-FPI, SAC-PID, SAC-LagNet). SAC-ALaM, however, consistently learns safe and high-performing policies regardless of task difficulty. Figure \ref{fig:aggregate_performance_profile} compares the normalized performance across all environments. To evaluate asymptotic performance, we average the final 10\% of training steps, apply per-task Min-Max normalization, and compute the cross-environment means with 95\% confidence intervals. With the x-axis inverted (where further right indicates lower costs), SAC-ALaM firmly occupies the optimal region in the top-right corner with significantly narrower confidence intervals. This confirms that SAC-ALaM not only achieves the highest overall normalized performance but also demonstrates strong consistency across diverse tasks.

To highlight SAC-ALaM's ability to mitigate the training oscillations inherent to SAC-LagNet, Figure \ref{fig:oscillation_curves} visualizes their individual training dynamics. Specifically, SAC-LagNet relies on standard Lagrangian formulation and updates its multiplier network via conventional dual gradient ascent \eqref{eq:SAC-LagNet}. In contrast, SAC-ALaM uses augmented Lagrangian with a supervised regression-based multiplier update. As observed, while SAC-LagNet suffers from severe fluctuations that ultimately degrade the final policy return, SAC-ALaM maintains a highly stable optimization process. It incurs fewer constraint violations during training and converges to a safe and effective policy.

\begin{figure*}[htbp]
    \centering

    \begin{minipage}{0.24\textwidth}
        \includegraphics[width=\linewidth]
        {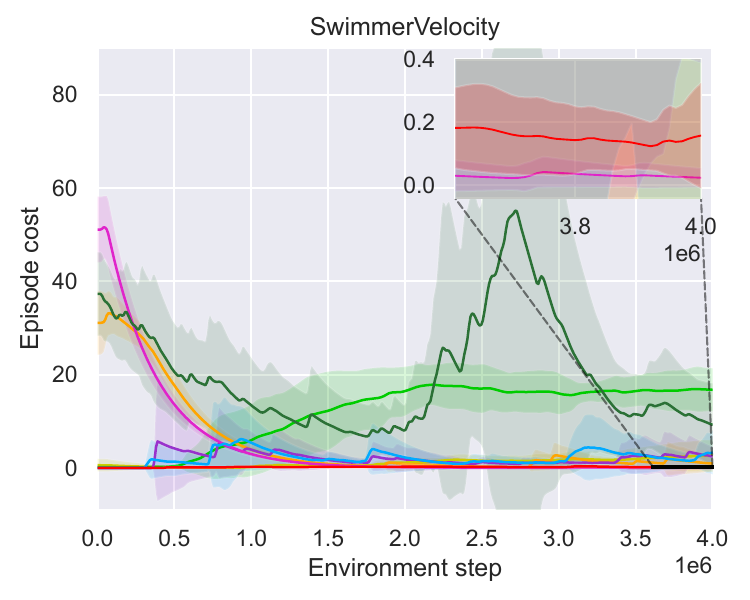}
    \end{minipage}\hfill
    \begin{minipage}{0.24\textwidth}
        \includegraphics[width=\linewidth]
        {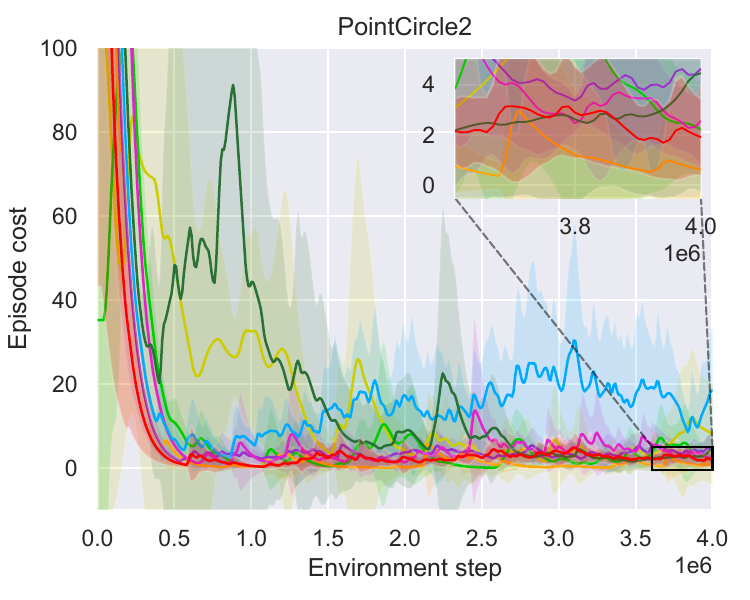}
    \end{minipage}\hfill
    \begin{minipage}{0.24\textwidth}
        \includegraphics[width=\linewidth]
        {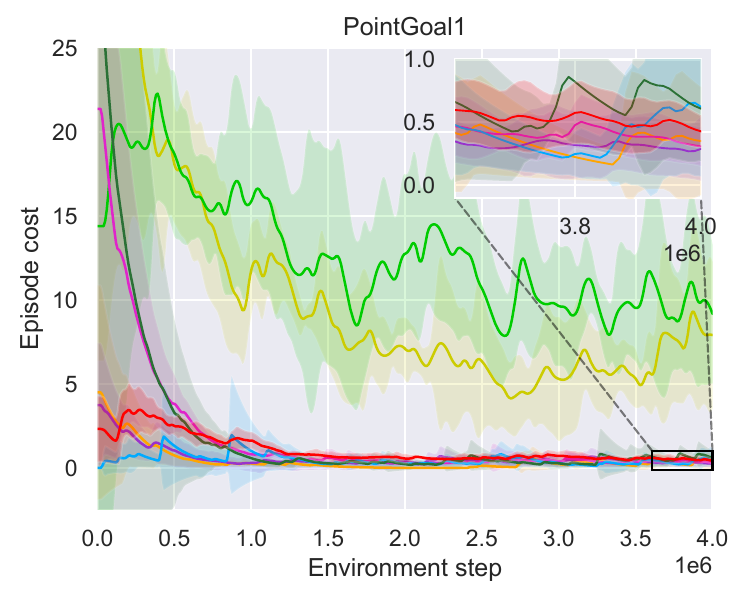}
    \end{minipage}\hfill
    \begin{minipage}{0.24\textwidth}
        \includegraphics[width=\linewidth]
        {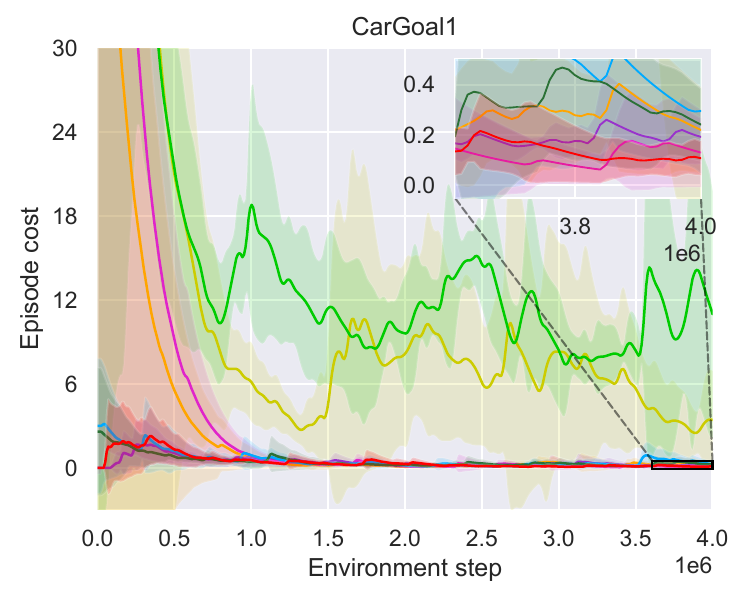}
    \end{minipage}

    \vspace{-0.0cm}
    
    \begin{minipage}{0.24\textwidth}
        \includegraphics[width=\linewidth]{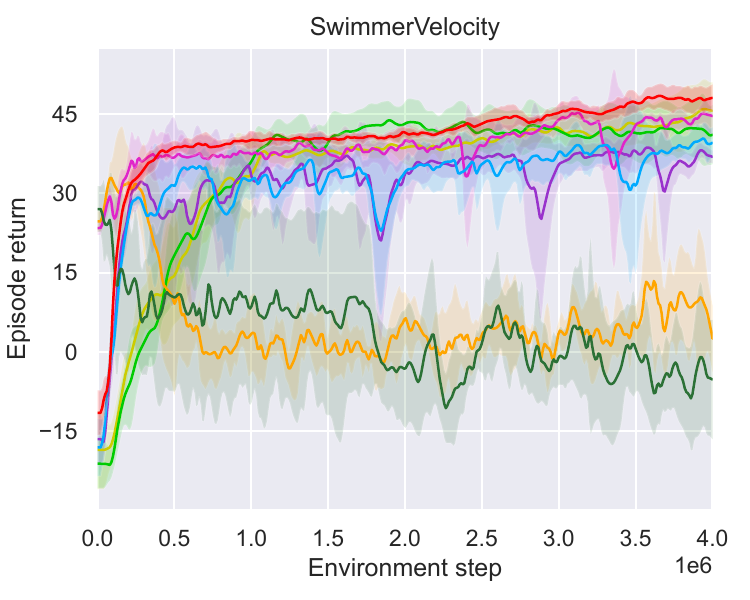}
    \end{minipage}\hfill
    \begin{minipage}{0.24\textwidth}
        \includegraphics[width=\linewidth]
        {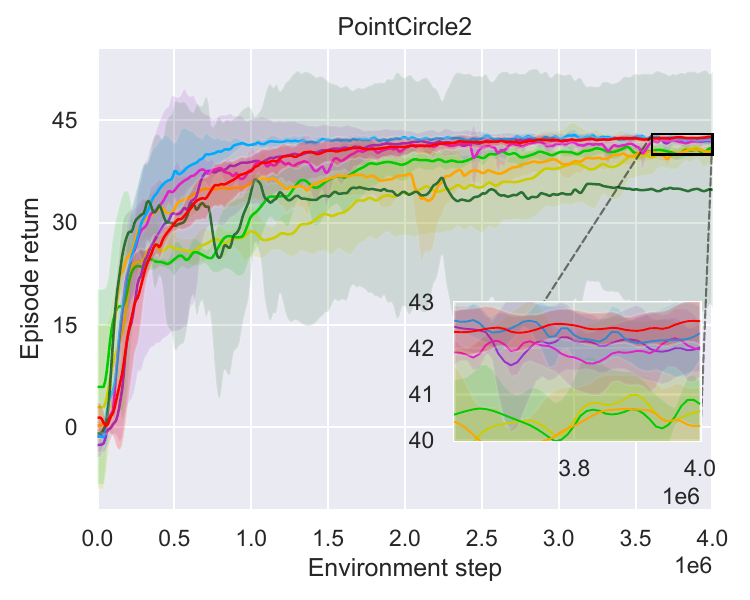}
    \end{minipage}\hfill
    \begin{minipage}{0.24\textwidth}
        \includegraphics[width=\linewidth]
        {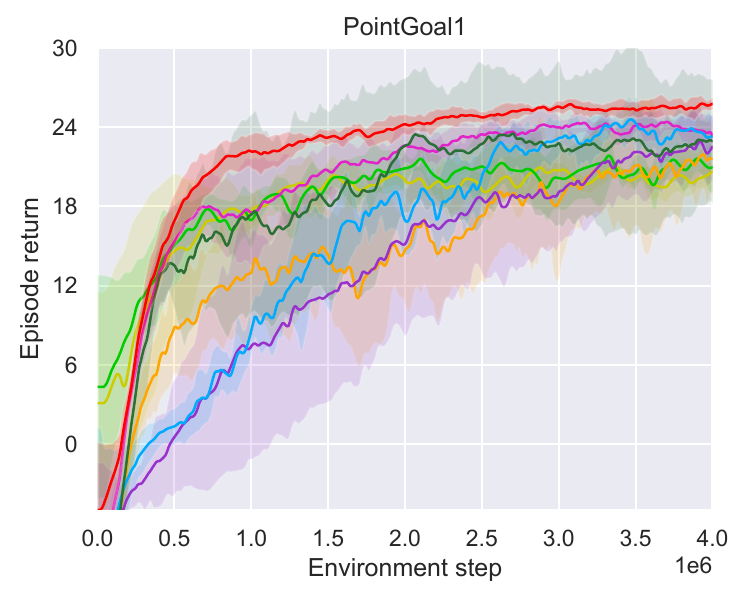}
    \end{minipage}\hfill
    \begin{minipage}{0.24\textwidth}
        \includegraphics[width=\linewidth]
        {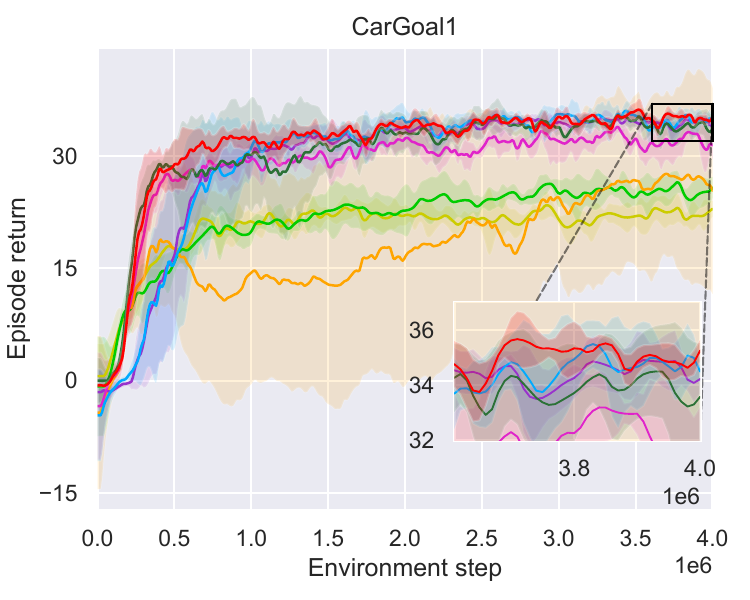}
    \end{minipage}

    \vspace{-0.0cm}
    
    \begin{minipage}{0.24\textwidth}
        \includegraphics[width=\linewidth]{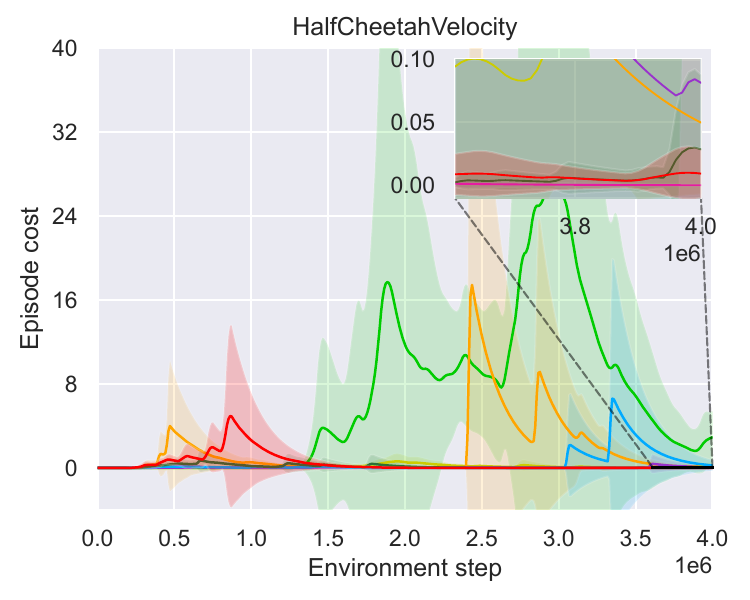}
    \end{minipage}\hfill
    \begin{minipage}{0.24\textwidth}
        \includegraphics[width=\linewidth]{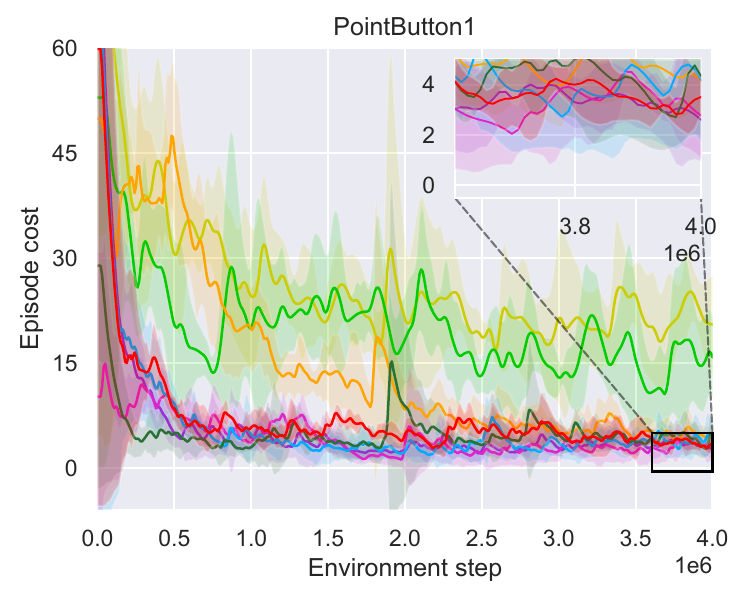}
    \end{minipage}\hfill
    \begin{minipage}{0.24\textwidth}
        \includegraphics[width=\linewidth]{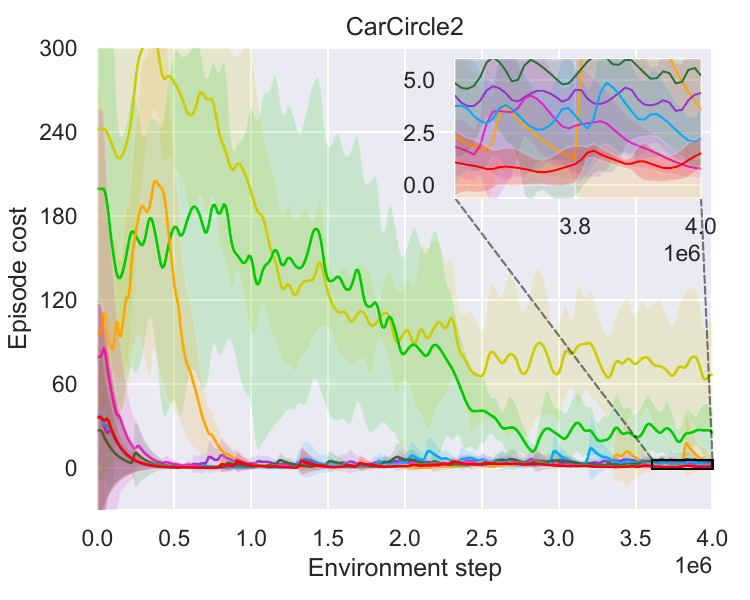}
    \end{minipage}\hfill
    \begin{minipage}{0.24\textwidth}
        \includegraphics[width=\linewidth]{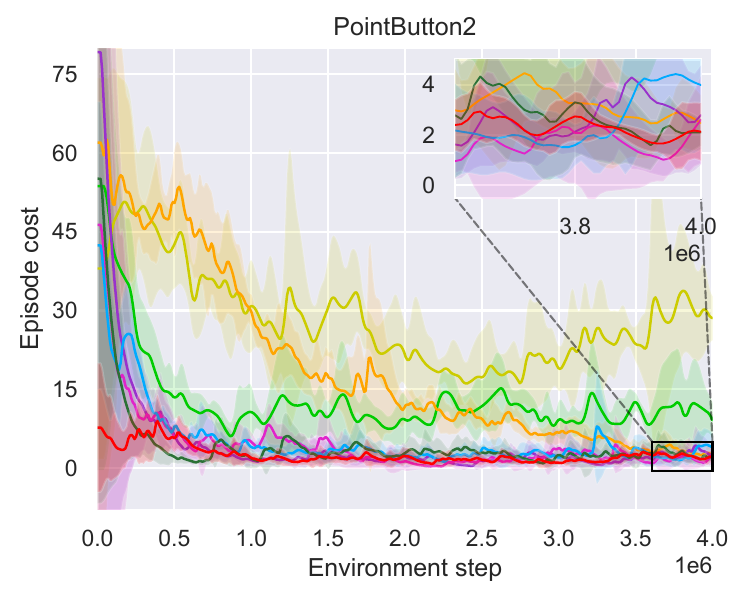}
    \end{minipage}

    \vspace{-0.0cm}
    
    \begin{minipage}{0.24\textwidth}
        \includegraphics[width=\linewidth]{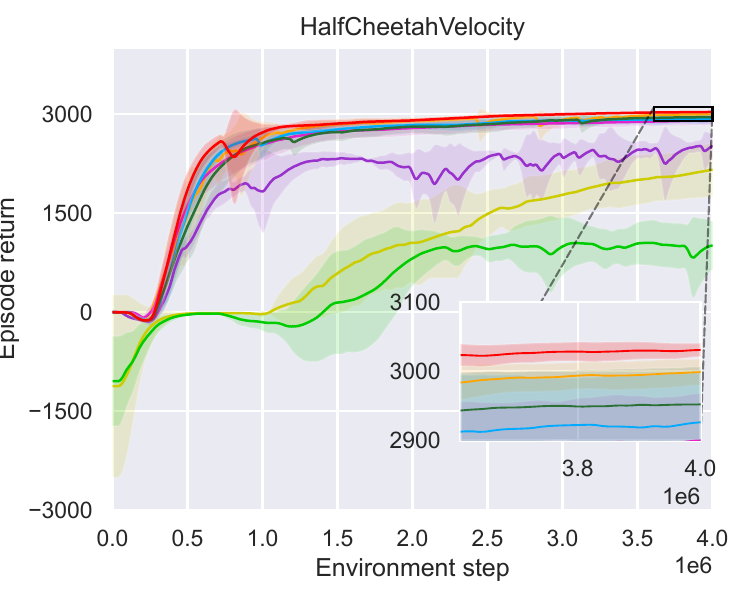}
    \end{minipage}\hfill
    \begin{minipage}{0.24\textwidth}
        \includegraphics[width=\linewidth]{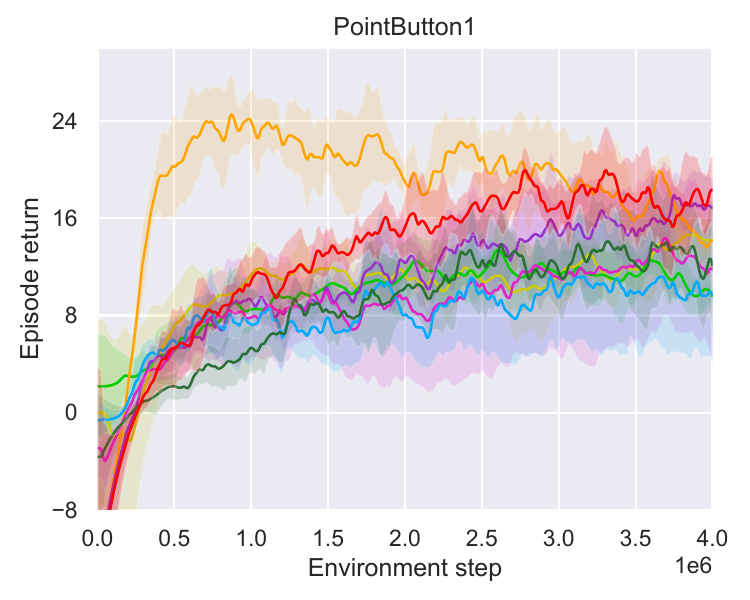}
    \end{minipage}\hfill
    \begin{minipage}{0.24\textwidth}
        \includegraphics[width=\linewidth]{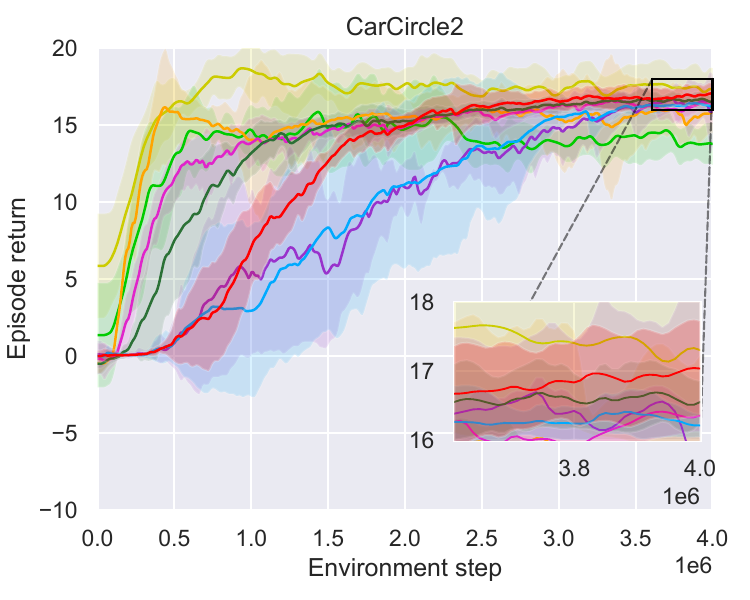}
    \end{minipage}\hfill
    \begin{minipage}{0.24\textwidth}
        \includegraphics[width=\linewidth]{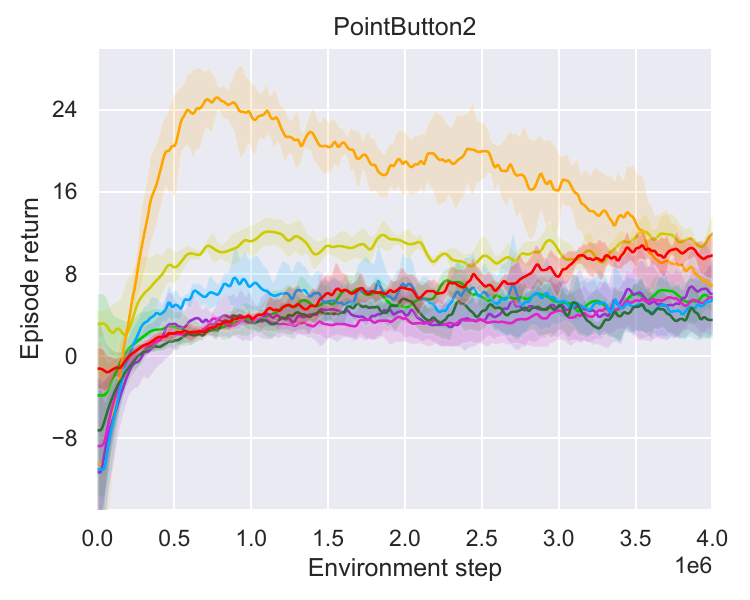}
    \end{minipage}

    \vspace{-0.0cm} 
    \includegraphics[width=0.8\textwidth]{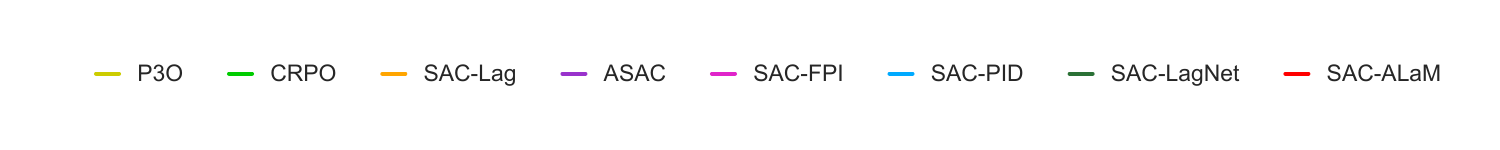} 
    \caption{Learning curves of SAC-ALaM and baselines across 8 environments. The solid lines represent the average performance over 5 seeds, and the shaded areas indicate the 95\% confidence intervals. The first row shows the average episode cost, and the second row shows the average episode return.}
    \label{fig:training_curves}
\end{figure*}

\begin{figure*}[htbp]
    \centering
    \begin{minipage}{0.24\textwidth}
        \includegraphics[width=\linewidth]{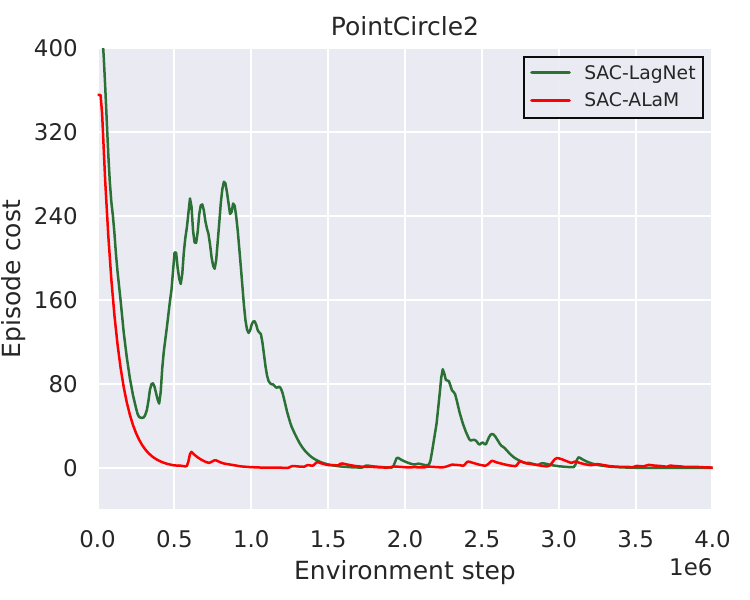}
    \end{minipage}\hfill
    \begin{minipage}{0.24\textwidth}
        \includegraphics[width=\linewidth]{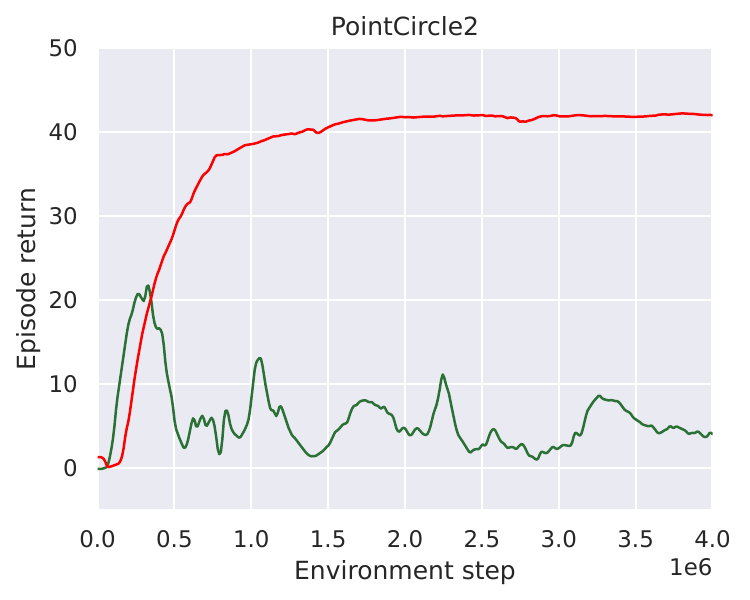}
    \end{minipage}\hfill
    \begin{minipage}{0.24\textwidth}
        \includegraphics[width=\linewidth]{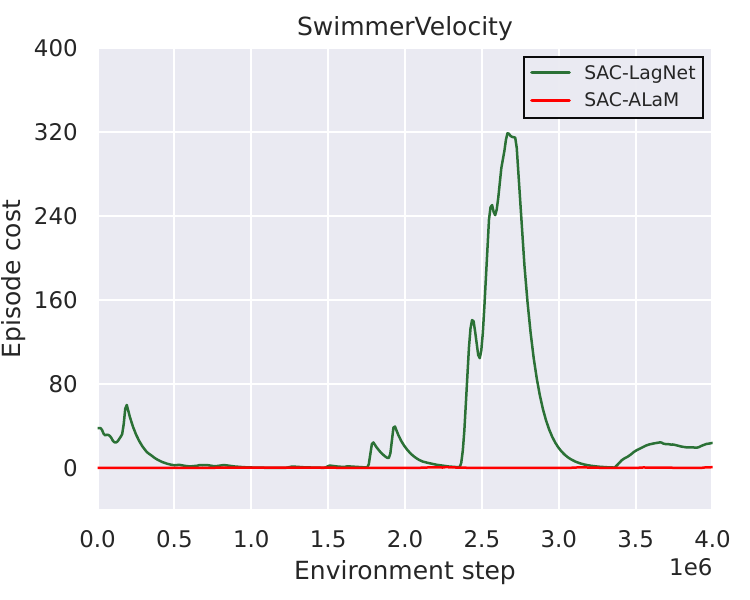}
    \end{minipage}\hfill
    \begin{minipage}{0.24\textwidth}
        \includegraphics[width=\linewidth]{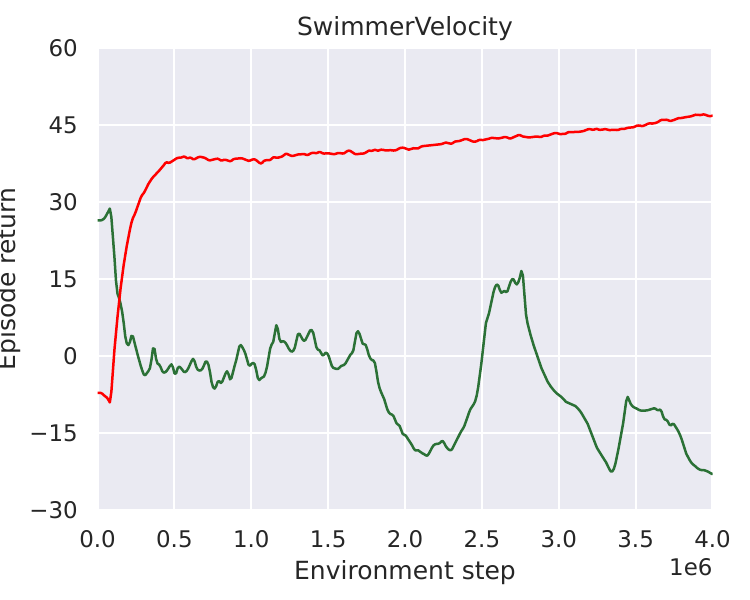}
    \end{minipage}
    \caption{Training stability comparison in the PointCircle2 and SwimmerVelocity tasks.}
    \label{fig:oscillation_curves}
\end{figure*}


To demonstrate the advantage of dual supervised regression, Figure \ref{fig:ablation_curves} presents an ablation study comparing SAC-ALaM against SAC-ALaM-GA, a variant utilizing traditional dual gradient ascent. As illustrated, SAC-ALaM-GA suffers from more training oscillations. This discrepancy arises because standard gradient ascent merely adjusts parameters based on local constraint violations, failing to approximate the theoretical dual target. This violates the theoretical convergence requirements of augmented Lagrangian methods. Furthermore, owing to the neural network's inherent generalization, improper local gradient updates inevitably corrupt the multiplier's estimations in neighboring regions, ultimately degrading the agent's overall performance. Conversely, by explicitly fitting the theoretical target, SAC-ALaM consistently outperforms its gradient ascent counterpart in both task efficacy and safety adherence.

\begin{figure*}[htbp]
    \centering
    \begin{minipage}{0.33\textwidth}
        \includegraphics[width=\linewidth]{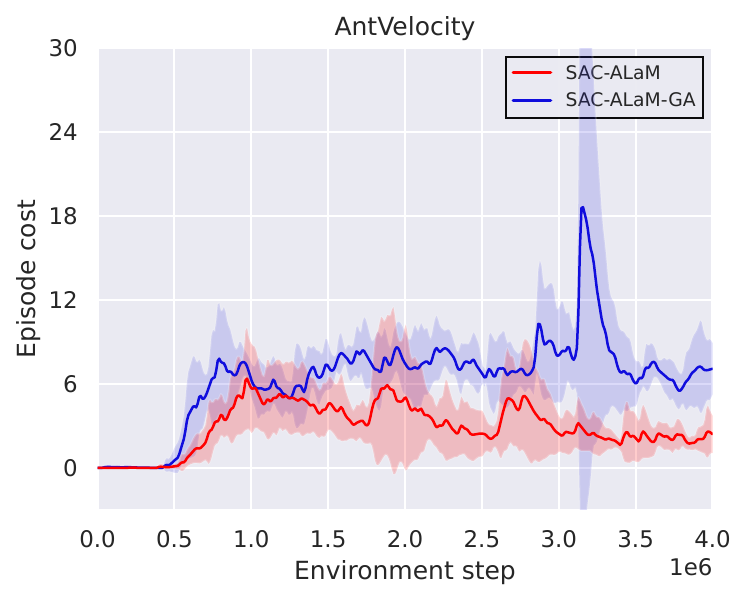}
    \end{minipage}\hfill
    \begin{minipage}{0.33\textwidth}
        \includegraphics[width=\linewidth]{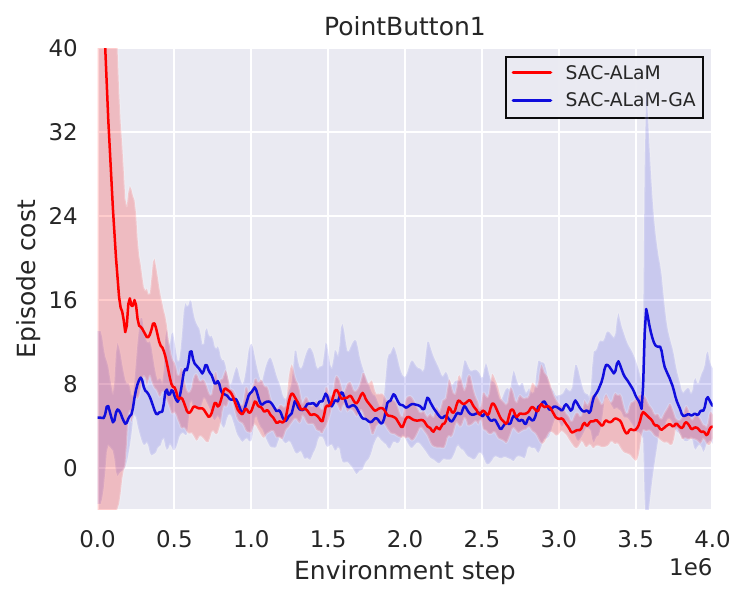}
    \end{minipage}\hfill
    \begin{minipage}{0.33\textwidth}
        \includegraphics[width=\linewidth]{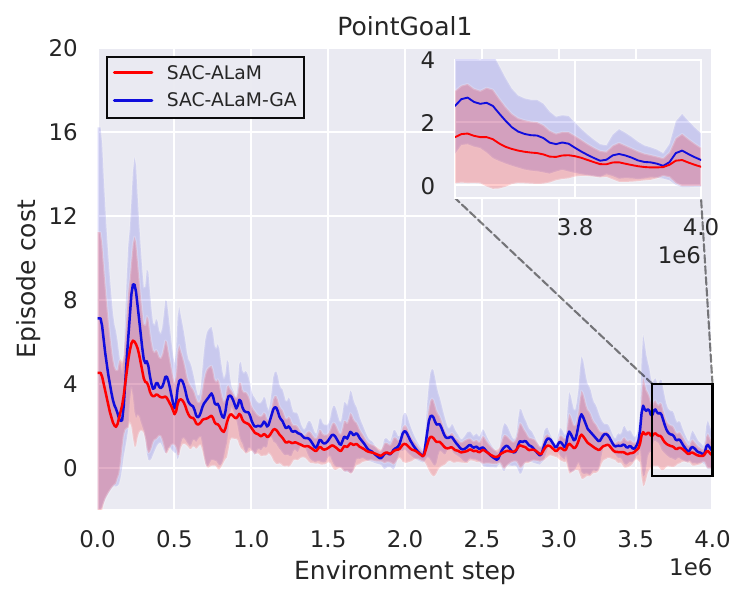}
    \end{minipage}

    \vspace{0.1cm}
    
    \begin{minipage}{0.33\textwidth}
        \includegraphics[width=\linewidth]{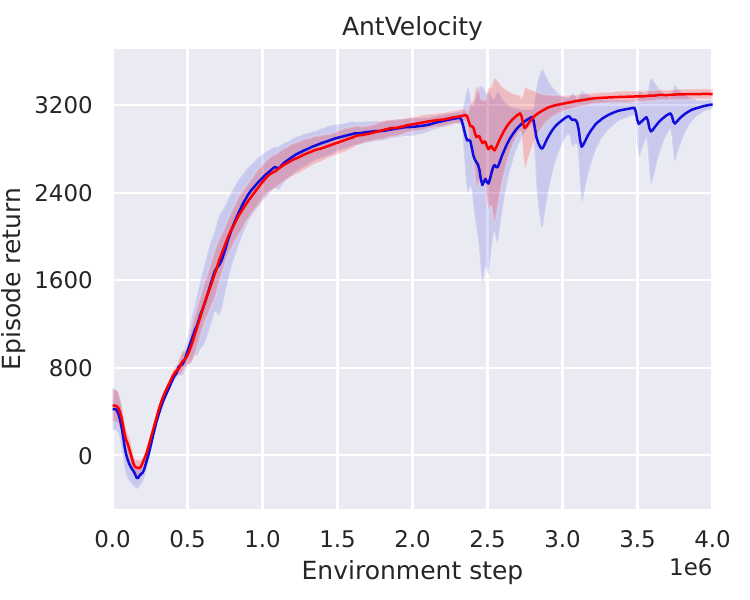}
    \end{minipage}\hfill
    \begin{minipage}{0.33\textwidth}
        \includegraphics[width=\linewidth]{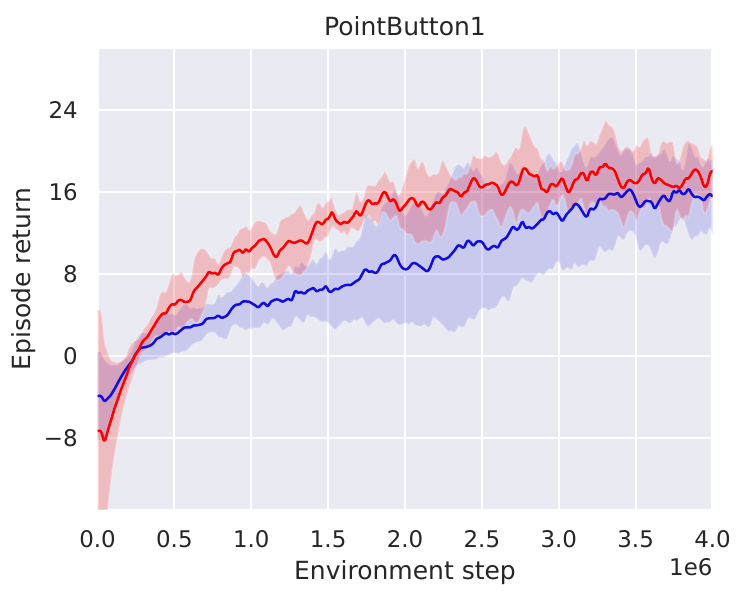}
    \end{minipage}\hfill
    \begin{minipage}{0.33\textwidth}
        \includegraphics[width=\linewidth]{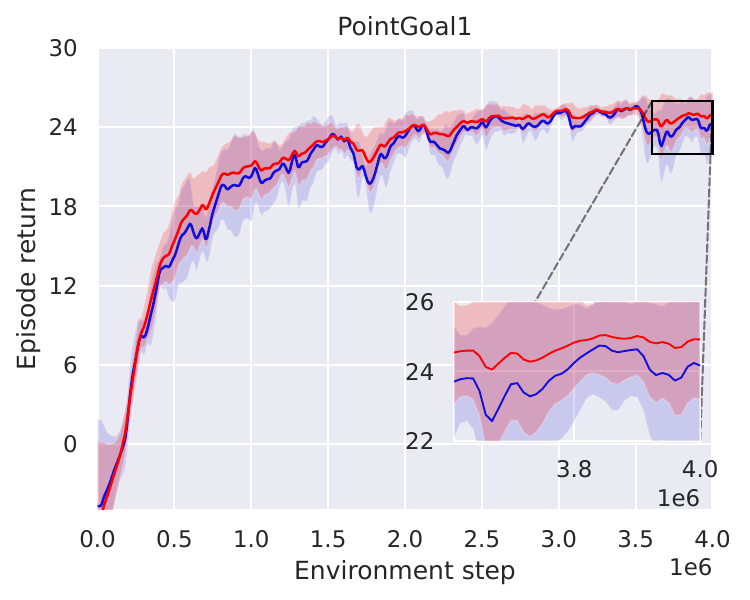}
    \end{minipage}
    
    \caption{Ablation of multiplier update. Solid lines and shaded regions denote the mean and 95\% confidence intervals across 5 seeds. The first row shows the average episode cost, and the second row shows the average episode return. }
    \label{fig:ablation_curves}
\end{figure*}

Figure \ref{fig:heatmaps} visualizes the multiplier functions to illustrate risk identification capabilities.
In the Goal task, the agent must reach the green target while avoiding blue obstacles. SAC-ALaM effectively distinguishes hazards from safe zones. As velocity increases, multipliers inside the obstacles decrease due to shorter departure time, while these values in the surrounding areas increase for extended braking distance. In contrast, SAC-LagNet struggles to clearly separate safe from dangerous regions, collapsing into a uniform high-risk estimation at high speeds. Notably, despite extreme velocities being rarely explored during training, SAC-ALaM successfully provides accurate safety feedback.
In the Circle task, the agent must navigate a green circular path within the grey boundary. SAC-ALaM's multiplier exhibits spatial symmetry when stationary, aligning with the underlying physical dynamics. Upon introducing eastward velocities, the network shifts high-risk estimations to the right-hand regions. Specifically, the high multiplier values are concentrated in the lower-right quadrant, where eastward momentum and counter-clockwise steering elevate collision probability; as the agent decelerates in the top-right area, the corresponding risk decreases. SAC-LagNet fails to capture these velocity-dependent risks. 
Overall, these visualizations confirm that the SAC-ALaM trains a physically consistent and robustly generalizable multiplier model for safety monitoring.

\begin{figure*}[htbp]
    \centering
    
    \begin{minipage}{0.48\textwidth}
        \includegraphics[width=\linewidth]{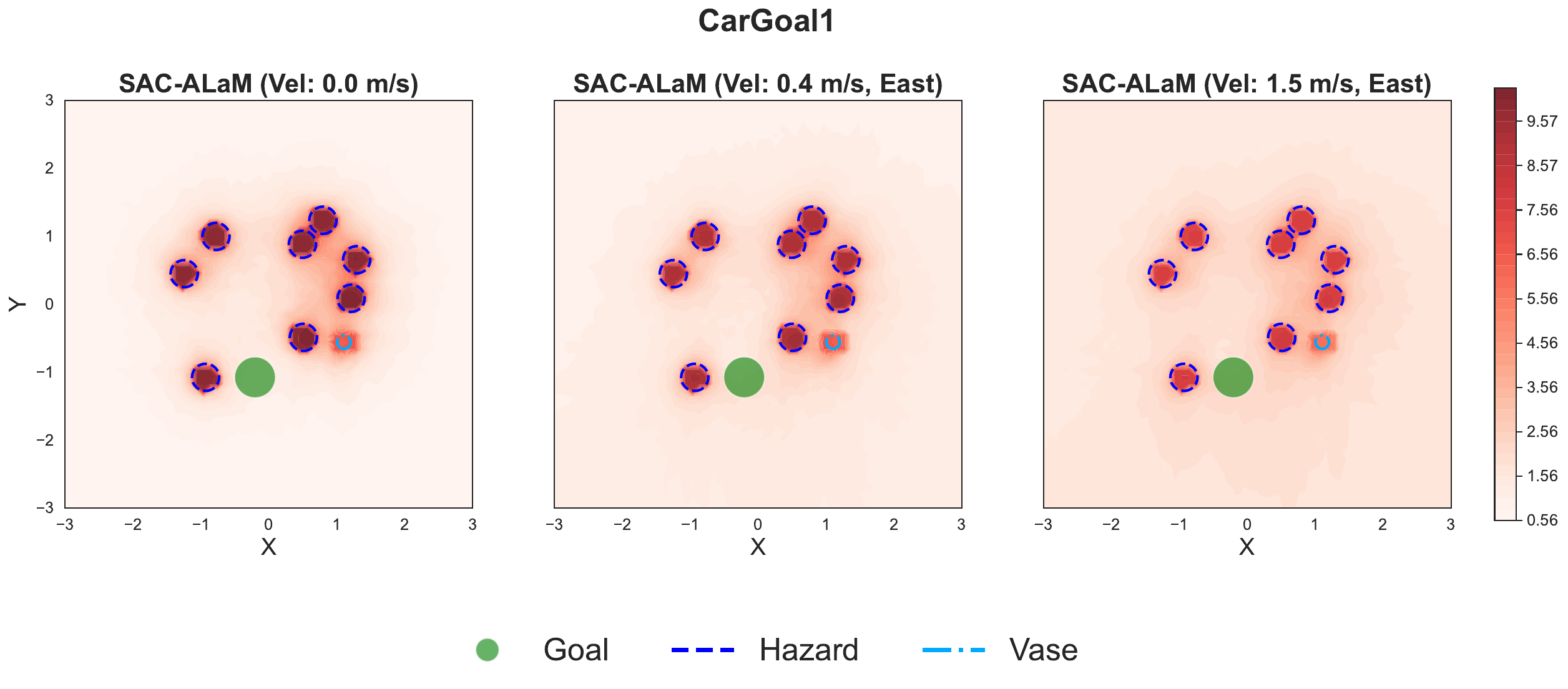}
    \end{minipage}\hfill
    \begin{minipage}{0.48\textwidth}
        \includegraphics[width=\linewidth]{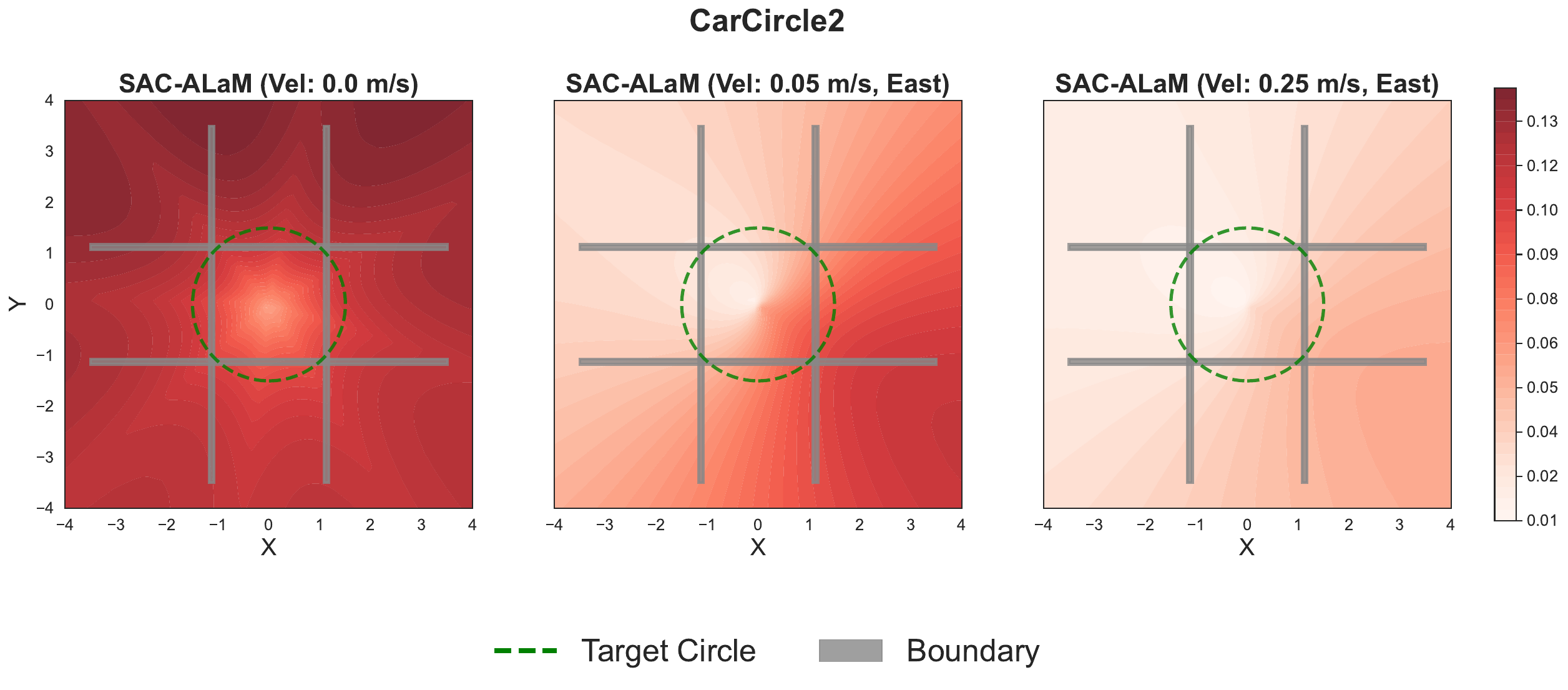}
    \end{minipage}

    \vspace{0.1cm}
    
    \begin{minipage}{0.48\textwidth}
        \includegraphics[width=\linewidth]{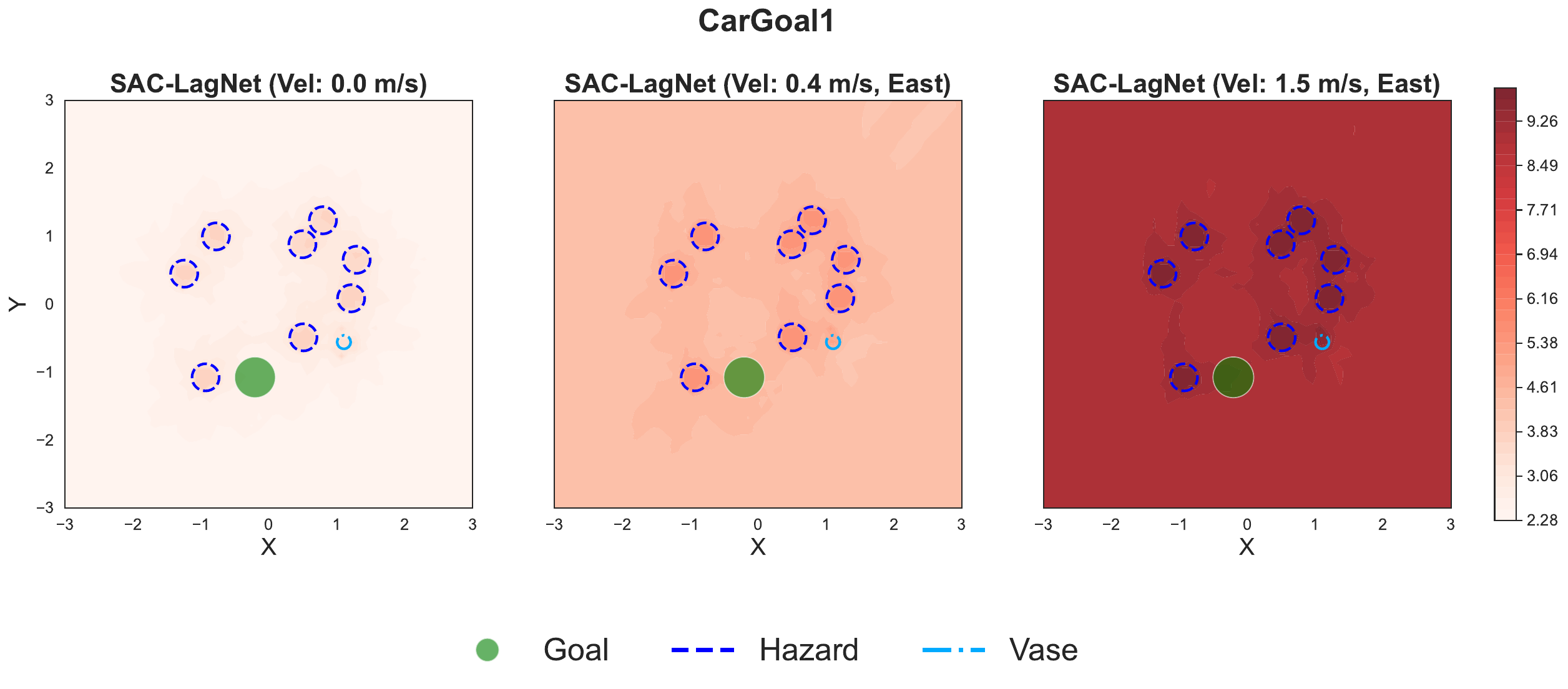}
    \end{minipage}\hfill
    \begin{minipage}{0.48\textwidth}
        \includegraphics[width=\linewidth]{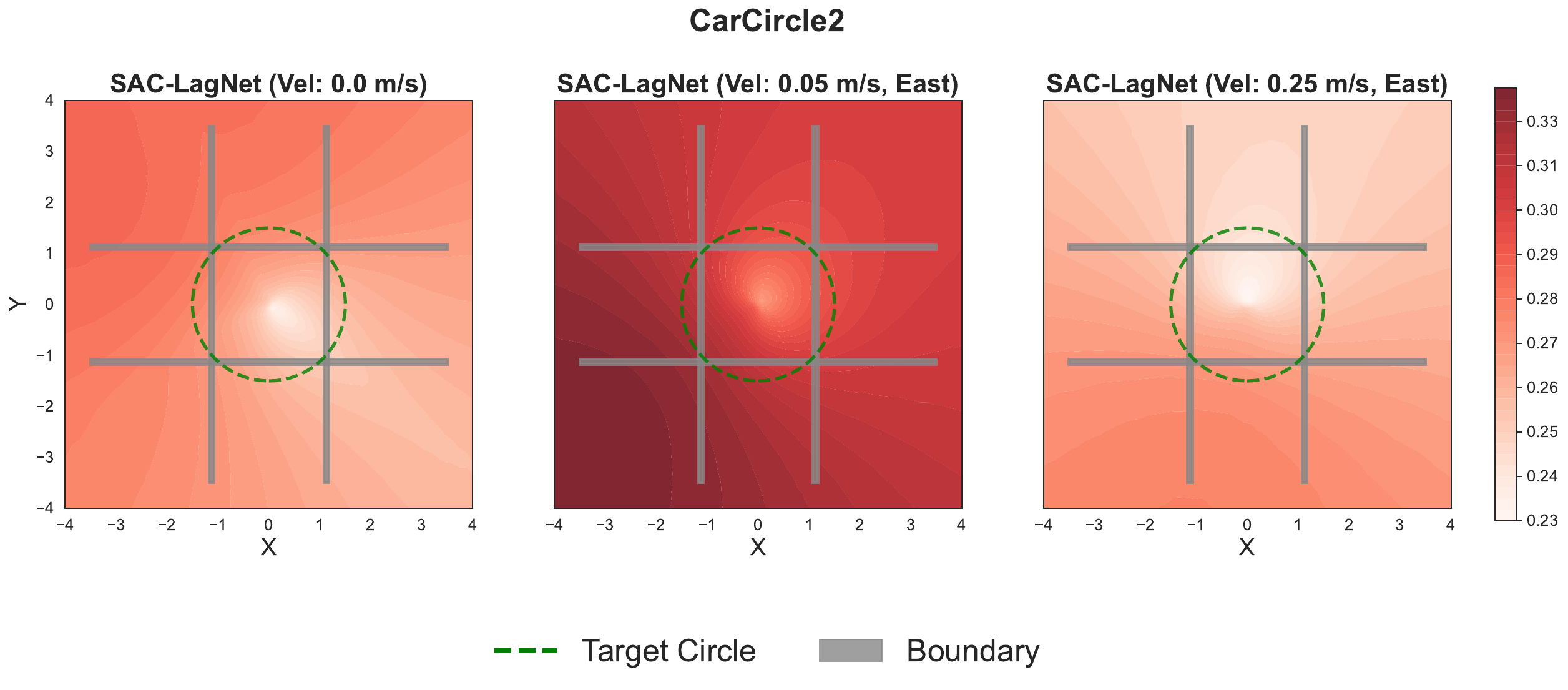}
    \end{minipage}
    
    \caption{Heatmaps of the multiplier for SAC-ALaM and SAC-LagNet across varying agent velocities.}
    \label{fig:heatmaps}
\end{figure*}

\section{Conclusion}
In this paper, we propose the ALaM framework for stable learning of state-wise multipliers in safe RL. By introducing an augmented Lagrangian with supervised regression to update the multiplier, ALaM stabilizes the training dynamics, with provable sequential convergence to an optimal policy. We further develop a practical safe RL algorithm, SAC-ALaM, by combining ALaM with soft actor-critic. Empirically, SAC-ALaM yields high-performing, safety-constrained policies while providing a well-calibrated multiplier for risk assessment.

\bibliographystyle{IEEEtran}
\bibliography{sac-ALaM}

\begin{IEEEbiography}[{\includegraphics[width=1in,height=1.25in,clip,keepaspectratio]{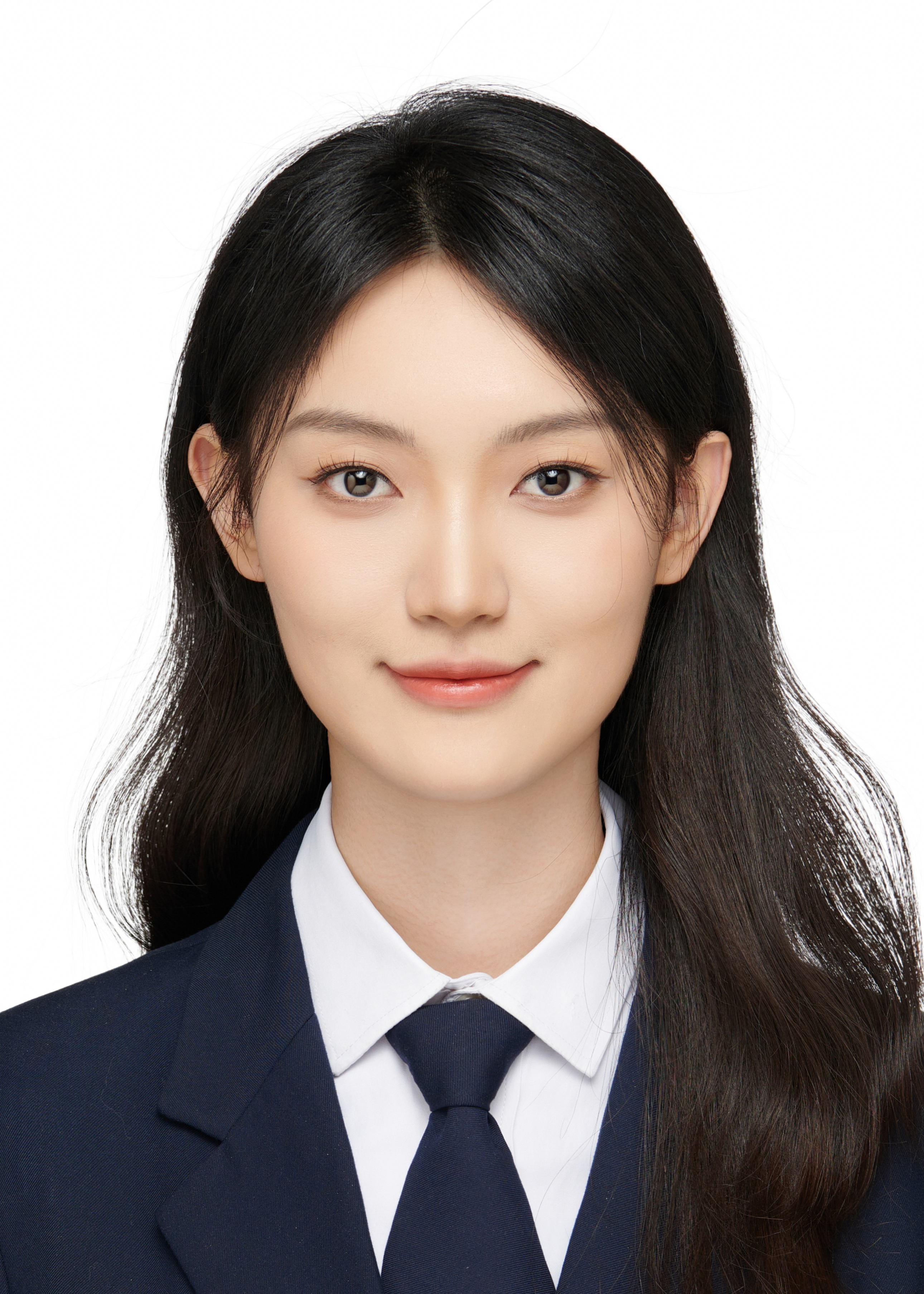}}]{Jiaming Zhang}
Jiaming Zhang received her B.S. degree in mathematics and applied mathematics from the School of Mathematics, Shandong University, Jinan, China in 2024. She is currently pursuing her Ph.D. degree in the Department of Mathematical Sciences, Tsinghua University, Beijing, China. Her research interests include continuous optimization and safe reinforcement learning.
\end{IEEEbiography}

\vspace{-10pt}

\begin{IEEEbiography}[{\includegraphics[width=1in,height=1.25in,clip,keepaspectratio]{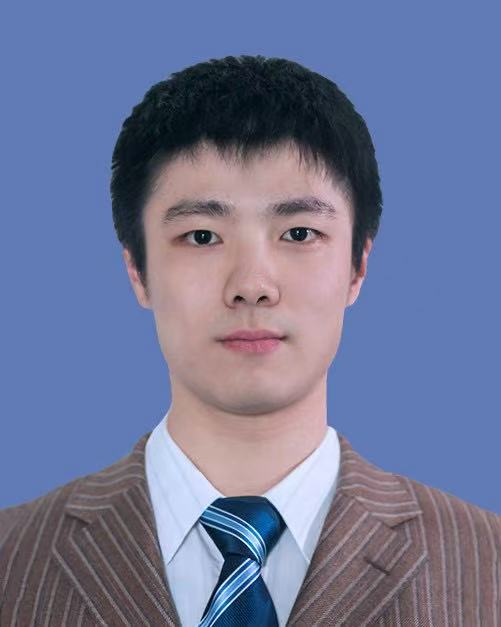}}]{Yujie Yang}
Yujie Yang received his B.S. degree in automotive engineering from the School of Vehicle and Mobility, Tsinghua University, Beijing, China in 2021. He is currently pursuing his Ph.D. degree in the School of Vehicle and Mobility, Tsinghua University, Beijing, China. His research interests include safe reinforcement learning and decision and control of autonomous vehicles.
\end{IEEEbiography}

\begin{IEEEbiography}[{\includegraphics[width=1in,height=1.25in,clip,keepaspectratio]{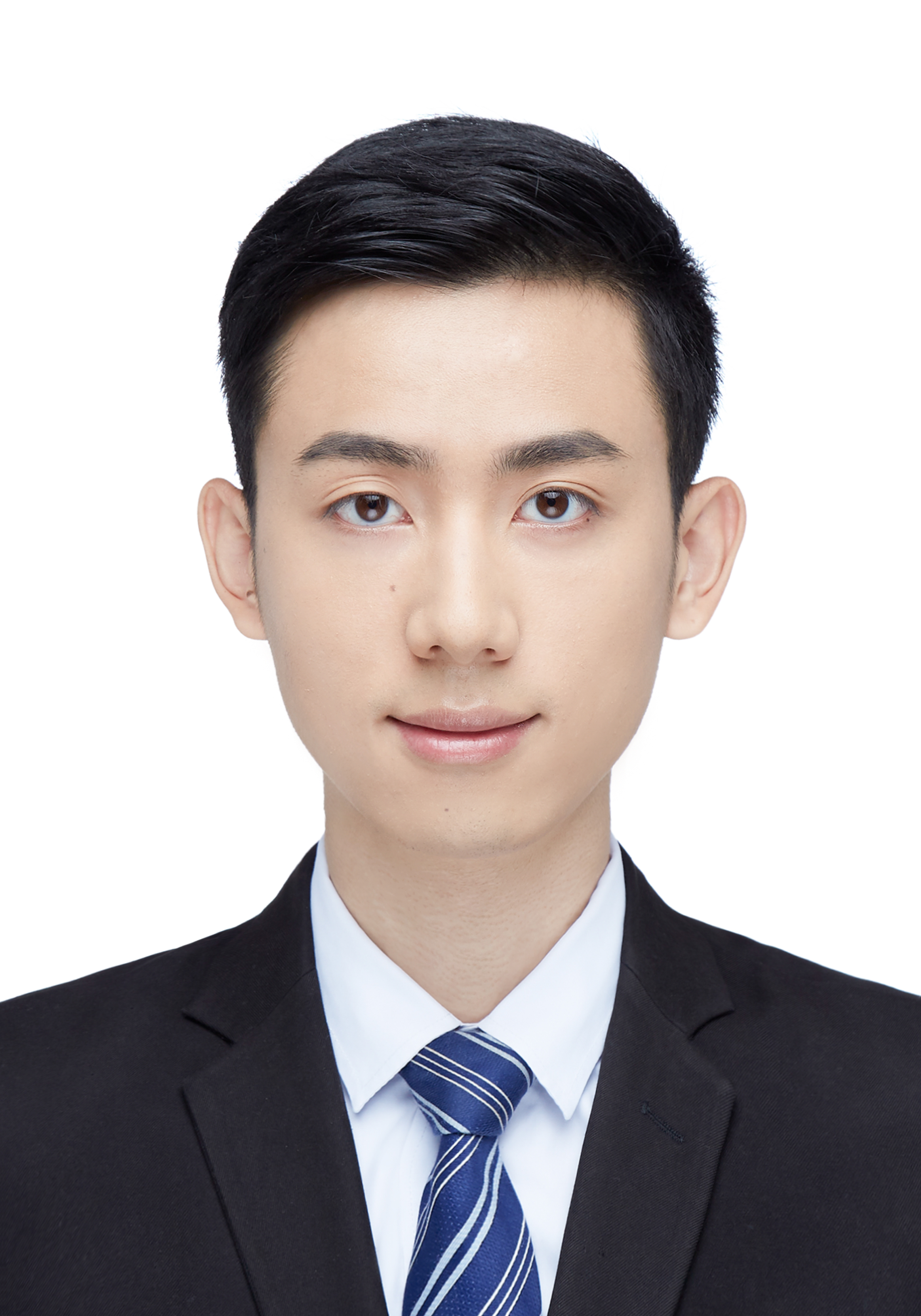}}]{Yao Lyu}
Yao Lyu received his B.Eng. degree in 2019 and his Ph.D. degree in 2025 from Tsinghua University, where he currently serves as a Postdoctoral Researcher in the School of Vehicle and Mobility. His active research interests include end-to-end autonomous driving, embodied artificial intelligence, deep reinforcement learning, and neural network optimization. He has authored over 20 peer-reviewed publications in top-tier venues and was awarded the CVCI 2023 Best Paper Award. Dr. Lyu actively contributes to the academic community, serving as a reviewer for IEEE TNNLS, IEEE Cyber, IEEE TITS, NeurIPS, ACC, CDC, etc.
\end{IEEEbiography}

\vspace{-10pt}

\begin{IEEEbiography}[{\includegraphics[width=1in,height=1.25in,clip,keepaspectratio]{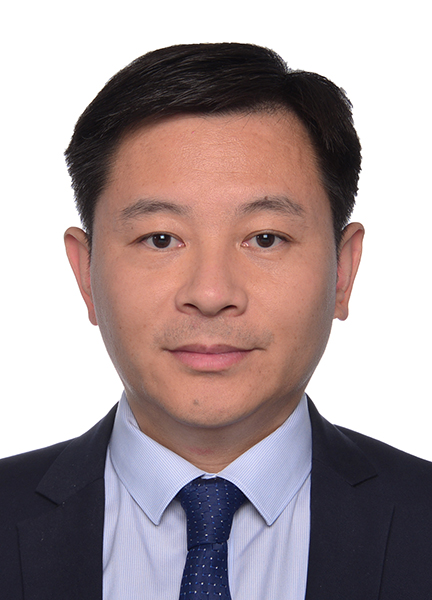}}]{Shengbo Eben Li}
Shengbo Eben Li (Senior Member, IEEE) received his M.S. and Ph.D. degrees from Tsinghua University in 2006 and 2009. He has worked at Stanford University, University of Michigan, and UC Berkeley. He is now a professor at Tsinghua University, working on intelligent vehicles and driver assistance, embodied intelligence for robotics, deep reinforcement learning, optimal control and estimation, etc. He is the author of over 250 peer-reviewed journal/conference papers, and co-inventor of over 40 patents. Dr. Li has received over 20 prestigious awards, including Youth Sci. \& Tech Award of Ministry of Education (annually 10 receivers in China), Natural Science Award of Chinese Association of Automation (First level), National Award for Progress in Sci \& Tech of China, and best (student) paper awards or finalists of IEEE ITSC, IEEE IVS, IET ITS, ICCAS, IFAC MECC, CAA CVCI, IEEE ICUS, CCCC, IEEE ITSM, L4DC, Automotive Innovation, etc. He was a member of Board Governor of IEEE ITS Society. He serves as the director of Technical Committee on AI of SAE-China, deputy director of Technical Committee on Vehicle Control and Intelligence of CAA, and the leader of AI working group in China Industry Innovation Alliance for ICVs.He also serves as Senior AE of IEEE OJ ITS, AEs of IEEE ITSM, IEEE TITS, IEEE TNNLS, IEEE TCST, and area chairs of ICLR and ICML, etc.
\end{IEEEbiography}

\vspace{-10pt}

\begin{IEEEbiography}[{\includegraphics[width=1in,height=1.25in,clip,keepaspectratio]{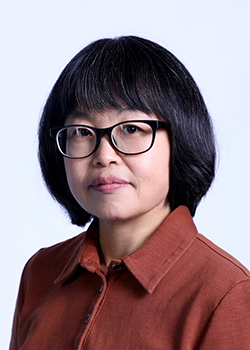}}]{Liping Zhang}
Liping Zhang is currently a tenured Professor in Department of Mathematical Sciences, Tsinghua University. She received her Ph.D. degree from the Academy of Mathematics and Systems Sciences, Chinese Academy of Sciences in 2001. Her research interests include continuous optimization, tensor analysis and computation, machine learning. She has published more than 70 research papers in international journals such as Mathematical Programming, SIAM Journal on Optimization, Mathematics of Computation, Mathematics of Operational Research, SIAM Journal on Matrix Analysis and Applications, Journal of Machine Learning Research, Expert Systems with Applications, etc.
\end{IEEEbiography}

\vfill

\end{document}